\def\year{2020}\relax
\documentclass[letterpaper]{article} 
\usepackage{aaai20}  
\usepackage{times}  
\usepackage{helvet} 
\usepackage{courier}  
\usepackage[hyphens]{url}  
\usepackage{graphicx} 
\usepackage{balance}
\usepackage{graphicx}  
\usepackage{graphicx}
\usepackage{amsmath,amsfonts,amssymb, amsthm}
\DeclareMathOperator*{\argmax}{arg\,max}

\newtheorem{theorem}{Theorem}

\newtheorem{lemma}[theorem]{Lemma}

\PassOptionsToPackage{hyphens}{url}\usepackage{hyperref}
\usepackage{subcaption}
\usepackage[ruled, linesnumbered]{algorithm2e}
\usepackage{multirow}
\usepackage{xcolor}
\usepackage{booktabs}
\usepackage{color, colortbl}
\title{Exponential Family Graph Embeddings}
\author{Abdulkadir \c{C}elikkanat\\
CentraleSup\'{e}lec and Inria Saclay\\
University of Paris-Saclay\\
Gif-Sur-Yvette, France\\
abdulkadir.celikkanat@centralesupelec.fr
\And
Fragkiskos D. Malliaros\\
CentraleSup\'{e}lec and Inria Saclay\\
University of Paris-Saclay\\
Gif-Sur-Yvette, France\\
fragkiskos.malliaros@centralesupelec.fr
}

\begin{document}

\maketitle

\begin{abstract}
Representing networks in a low dimensional latent space is a crucial task with many interesting applications in graph learning problems, such as link prediction and node classification. A widely applied network representation learning paradigm is based on the combination of random walks for sampling context nodes and the traditional \textit{Skip-Gram} model to capture center-context node relationships. In this paper, we emphasize on exponential family distributions to capture rich interaction patterns between nodes in random walk sequences. We introduce the generic  \textit{exponential family graph embedding} model, that generalizes random walk-based network representation learning techniques to exponential family conditional distributions. We study three particular instances of this model, analyzing their properties and showing their relationship to existing unsupervised learning models. Our experimental evaluation on real-world datasets demonstrates that the proposed techniques outperform well-known baseline methods in two downstream machine learning tasks.
\end{abstract}

\section{Introduction}\label{sec:introduction}
Graphs or networks have become ubiquitous as data from diverse disciplines can naturally be represented as graph
structures. Characteristics examples include social, collaboration, information and biological networks, or even
networks that are generated by textual information. Besides, graphs are not only useful as models for data representation but can be proven valuable in prediction and learning tasks. For example, one might wish to recommend
new friendship relationships in social networks such as Facebook and LinkedIn, predict the missing structure or
the role of a protein in a protein-protein interaction graph, or even to discover missing relations between entities in
a knowledge graph. To that end, the tasks of learning and analyzing large-scale real-world graph data drive
several important applications, but also pose a plethora of  challenges.

The major challenge in machine learning on graphs is how to incorporate information about its structure in the learning model. For example, in the case of friendship recommendations in social networks (also known as the \textit{link prediction} problem), in order to determine whether two unlinked users are similar, we need to obtain an informative representation of the users and their proximity --- that potentially is not fully captured by graph statistics (e.g., centrality criteria) \cite{chakr}, kernel functions \cite{vishwanathan2010graph}, or more generally other handcrafted features extracted from the graph \cite{Liben-Nowell:2007}. To deal with these challenges, a recent paradigm in network analysis, known as \textit{network representation learning} (NRL), aims at finding vector representations of nodes (i.e., \textit{node embeddings}), in such a way that the structure of the network and its various properties are preserved in lower-dimensional representation space.  The problem is typically expressed as an optimization task, where the goal is to optimize the mapping function from the graph space to a low-dimensional space, so that proximity relationships in the learned space reflect the structure of the original graph \cite{survey_hamilton_rex_leskovec,survey_goyal_ferrara,survey_cai_zheng_chen}. Furthermore, in most cases, the feature learning approach is purely \textit{unsupervised}. That way, after obtaining embeddings, the learned features can further be used in any downstream machine learning task, such as classification and prediction.

Initial studies in network representation learning  mostly focused on classical dimensionality reduction techniques via matrix factorization (e.g., \cite{grarep,relational_learning_social_dim}). Although the success of such approaches in capturing the structural properties of a network, they tend to be, unfortunately, not efficient for large scale networks.  Therefore, the community has concentrated on developing alternative approaches, and inspired by the field of natural language processing (NLP) \cite{word2vec}, \textit{random-walk based} methods have become a prominent line of research for network representation learning. Characteristic examples of such approaches constitute \textsc{DeepWalk} \cite{deepwalk} and \textsc{Node2Vec} \cite{node2vec} models. Typically, those methods sample a set of random walks from the input graph, treating them as the equivalent of sentences in natural language, while the nodes visited by the walk are considered as the equivalent of words. The point that essentially differentiates these methods, concerns the strategy that is followed to generate (i.e., sample) node sequences. The idea mainly aims to model \textit{center-context} node relationships, examining the occurrence of a node within a certain distance with respect to another node (as indicated by the random walk);  this information is then utilized to represent the relationship between a pair of nodes. Then, widely used NLP models, such as the \textit{Skip-Gram} model \cite{word2vec}, are used to learn node latent representations, examining \textit{simple co-occurrence relationships} of nodes within the set of random walk sequences.

Nevertheless, \textit{Skip-Gram} models the conditional distribution of nodes within a random walk based on the \textit{softmax} function, which might prohibit to capture richer  types of interaction patterns among nodes that co-occur within a random walk. Motivated by the aforementioned limitation of current random walk-based NRL methodologies, we argue that considering more expressive \textit{conditional probability models} to relate nodes within a random walk sequence, might lead to more informative latent node representations.

In particular, we capitalize on \textit{exponential family distribution} models to capture interactions between nodes in random walks. Exponential families correspond to a mathematically convenient parametric set of probability distributions, which is flexible in representing relationships among entities. More precisely, we introduce the concept of \textit{exponential family graph embeddings} (EFGE), that generalizes random walk NRL techniques to exponential family conditional distributions. We study three particular instances of the proposed \textsc{EFGE} model  that correspond to widely known exponential family distributions, namely the Bernoulli, Poisson and Normal distributions. The extensive experimental evaluation of the proposed models in the tasks of node classification and link prediction suggests that, the proposed \textsc{EFGE} models can further improve the predictive capabilities of node embeddings, compared to  traditional \textit{Skip-Gram}-based and other baseline methods. In addition, we further study  the objective function of the proposed parametric models,  providing connections to well-known unsupervised graph learning  models under appropriate parameter settings.

\vspace{.15cm}
\noindent \textbf{Contributions.}  The main contributions of the paper can be summarized as follows:
\begin{itemize}
    \item We introduce a novel representation learning model, called \textsc{EFGE}, which generalizes classical \textit{Skip-Gram}-based approaches to exponential family distributions, towards more expressive NRL methods that rely on random walks. We study three instances of the model, namely the \textsc{EFGE-Bern}, \textsc{EFGE-Pois} and \textsc{EFGE-Norm} models, that correspond to well-known distributions. 
    
    
    \item We show that the objective functions of  existing unsupervised and representation learning models, including word embedding in NLP \cite{word2vec} and overlapping community detection \cite{bigclam}, can be re-interpreted under the \textsc{EFGE} model.
    
    \item In a thorough experimental evaluation, we demonstrate that the proposed exponential family graph embedding models generally outperform  widely used baseline approaches in various learning tasks on graphs.  In addition, the running time to learn the representations  is similar to other \textit{Skip-Gram}-based models.
    
\end{itemize}


\vspace{.1cm}
\noindent \textbf{Source code.} The implementation of the proposed models is provided in the following website: \url{https://abdcelikkanat.github.io/projects/EFGE/}.

\section{Preliminary Concepts}\label{sec:preliminary} 

\subsection{Random Walk-based Node Embeddings}

Let $G=(\mathcal{V}, \mathcal{E})$ be a graph, where $\mathcal{V}$ and $\mathcal{E} \subseteq \mathcal{V}\times\mathcal{V}$ denote the vertex and  edge sets respectively. Random-walk based node embedding methods \cite{deepwalk,node2vec,our_tne,our_kernel_embedding} generate a set of node sequences by simulating random walks that can follow various strategies; node representations are then learned relying on these generated sequences. We use an ordered sequence of nodes $\textbf{w}=(w_1,...,w_L) \in \mathcal{W}$ to denote a \textit{walk}, if every pair $(w_{l-1}, w_{l})$ belongs to the edge set $\mathcal{E}$ for all $1 < l < L$. Then, the notation $\mathcal{W}$ will represent the set of walks of length $L$.

Being inspired from the the field of natural language processing and the \textit{Skip-Gram} model \cite{word2vec} for word embeddings, each walk is considered as a sentence in a document and similarly the surrounding vertices appearing at a certain distance from each node in a walk are defined as the \textit{context set} of this particular node, which is also called \textit{center} in our description. More formally, we will use $\mathcal{N}^{\textbf{w}}_{\gamma}(w_l) := \{w_{l+j}\in\mathcal{V}: -\gamma \leq j \leq \gamma, \ j \not= 0\}$ to denote the context sequence of  node $w_l$ in the random walk $\textbf{w}\in\mathcal{W}$.  Representations of nodes are  learned by optimizing the relationship between these center and context node pairs under a certain model. More formally, the objective function of  \textit{Skip-Gram} based models for network representation learning is defined in the following way:

\begin{align}\label{eq:skipgram_objective}
    \argmax_{\Omega} \frac{1}{N\!\cdot\! L} \sum_{\textbf{w} \in \mathcal{W}} \sum_{1\leq l \leq L}\sum_{v \in \mathcal{N}_{\gamma}(w_l)}\log p(y_{w_l,v}; \Omega),
\end{align}

\noindent where $y_{w_l,v_j}$ represents the relationship between the center $w_l$ and  context node $v$ in the walk $\textbf{w}\in \mathcal{W}$, $N$ is the number of walks, $L$ is length of walks and $\Omega=(\mathbf{\alpha}, \mathbf{\beta})$. Note that, we typically learn two embedding vectors $\mathbf{\alpha}[v]$ and $\mathbf{\beta}[v]$ for each node $v \in \mathcal{V}$, where $\mathbf{\beta}[v]$ corresponds to the vector if the node is interpreted as a center node  and $\mathbf{\alpha}[v]$ denotes the vector if  $v$ is considered as the context of other nodes. In all downstream machine learning applications, we only consider $\mathbf{\alpha}[v]$ to represent the embedding vector of $v$.

Generally speaking, random walk-based network representation learning methods can use different approaches to sample the context of a particular node. For instance, \textsc{DeepWalk} performs uniform truncated random walks, while \textsc{Node2Vec} is based on second order random walks to capture context information. Another crucial point related to  \textit{Skip-Gram} based models, has to do with the way that the relationship among center and context nodes in Eq. \eqref{eq:skipgram_objective} is modeled. In particular, \textsc{DeepWalk} uses the \textit{softmax} function to model the conditional distribution $p(\cdot)$ of a context node for a given center, in such a way that  nodes occurring in similar contexts tend to get close to each other in the latent representation space. In a similar way, \textsc{Node2Vec} adopts the negative sampling strategy, where it implicitly models co-occurrence relationships of context and center node pairs. As we will present shortly, in our work we rely on exponential family distributions, in order to  further extend and generalize random-walk NRL models to conditional probability distribution beyond the \textit{softmax} function --- towards  capturing  richer types of node interaction patterns.

\subsection{Exponential Families}
In this paragraph, we introduce the \textit{exponential family  distributions}, a parametric set of probability distributions that includes among others the Gaussian, Binomial and  Poisson distributions. A class of probability distributions is called exponential family distributions if they can be expressed as

\begin{align}\label{eq:exponential_form}
p(y) = h(y)\exp{\Big(\eta T(y) - A(\eta)\Big)},   
\end{align}

\noindent where $h$ is the \textit{base measure}, $\eta$ are the \textit{natural parameters}, $T$ is the \textit{sufficient statistic} of the distribution and $A(\eta)$ is  the \textit{log-normalizer} or \textit{log-partition} function \cite{exp-family}.  Different choices of base measure and sufficient statistics lead us to obtain different probability distributions. For instance, the base measure and sufficient statistic of the \textit{Bernoulli} distribution are $h(y)=1$ and $T(y)=y$ respectively, while for the \textit{Poisson} distribution  we have $h(y)=1/y!$ and $T(y)=y$.

As we mentioned above, exponential families contain a wide range of commonly used distributions,  providing a general class of models by re-parameterizing distributions in terms of the natural parameters $\eta$.  That way, we will use the natural parameter $\eta$ to design a set of network representation learning models, defining $\eta_{v,u}$ as the product of context and center vectors in the following way:
 \begin{align*}
     \eta_{v,u} := f\big(\mathbf{\beta}[v]^\top\!\cdot\!\mathbf{\alpha}[u]\big),
 \end{align*}
where $f$ is called the \textit{link function}. As we will present shortly in the following section, we have many alternative options for the form of the link function $f(\cdot)$.

\section{Proposed Approach}\label{sec:method}
In this section, we will introduce the proposed \textit{exponential family graph embedding} models, referred to as \textsc{EFGE}. The main idea behind this family of models is to utilize the expressive power of exponential family distribution towards conditioning context nodes with respect to the center node of interest. Initially, we will describe the formulation of the general objective function of the \textsc{EFGE} model, and then we will present particular instances of the model based on different exponential family distributions.

Let $\mathcal{W}$ be a collection of node sequences generated by following a random walk strategy over a given graph $G$, as defined in the previous section. Based on that, we can define a generic objective function to learn node embeddings in the following way: 

\begin{align}\label{eq:main_objective}
    \mathcal{L}(\mathbf{\alpha}, \mathbf{\beta)} :=  \argmax_{\Omega}\sum_{\textbf{w} \in \mathcal{W}} \sum_{1\leq l \leq L}\sum_{v \in \mathcal{V}}\log p(y_{w_l,v}^l; \Omega),
\end{align}

\noindent where $y_{w_l,v}^l$ is the observed value indicating the relationship between the center $w_l$ and  context node $v$. Here, we aim to find embedding vectors $\Omega=(\mathbf{\alpha}, \mathbf{\beta})$ by maximizing the log-likelihood function in Eq. \eqref{eq:main_objective}. Note that, the objective function in Eq. \eqref{eq:main_objective} is quite similar to the one of the \textit{Skip-gram} model \cite{word2vec} presented in Eq. \eqref{eq:skipgram_objective},  except that we also include nodes that are not belonging to context sets.

Instead of restricting ourselves to the \textit{Sigmoid} or \textit{Softmax} functions in order to model the probability in the objective function of Eq. \eqref{eq:main_objective}, we provide a generalization assuming that each $y_{w_l,v}$ follows an exponential family distribution. That way, the objective function used to learn node embedding vector sets $\Omega=(\mathbf{\alpha}, \mathbf{\beta})$ can be rewritten as follows:

\begin{align}\label{eq:main_objective_exp}
    \argmax_{\Omega}\sum_{\textbf{w} \in \mathcal{W}} \sum_{1\leq l \leq L}\sum_{v \in \mathcal{V}}\log h(y_{w_l,v}) &+ \eta_{w_l,v} T(y_{w_l,v})\nonumber \\ &- A(\eta_{w_l,v}).
\end{align}

\noindent As we can observe, Eq. \eqref{eq:main_objective_exp} which is the objective function of the generic \textsc{EFGE} graph embeddings model, generalizes \textit{Skip-Gram}-based models to exponential family conditional distributions described in Eq. \eqref{eq:exponential_form}. That way, \textsc{EFGE} models have the additional flexibility to utilize a wide range of exponential  distributions, allowing them to capture more complex types of node interactions beyond simple co-occurrence relationships. It is also important to stress out that, the first term of Eq. \eqref{eq:main_objective_exp} does not depend on  parameter $\eta_{w_l,v}$; this will bring an advantage during the  optimization process. 

\begin{figure*}[t]
\begin{subfigure}{.20\textwidth}
\centering
\includegraphics[width=\linewidth]{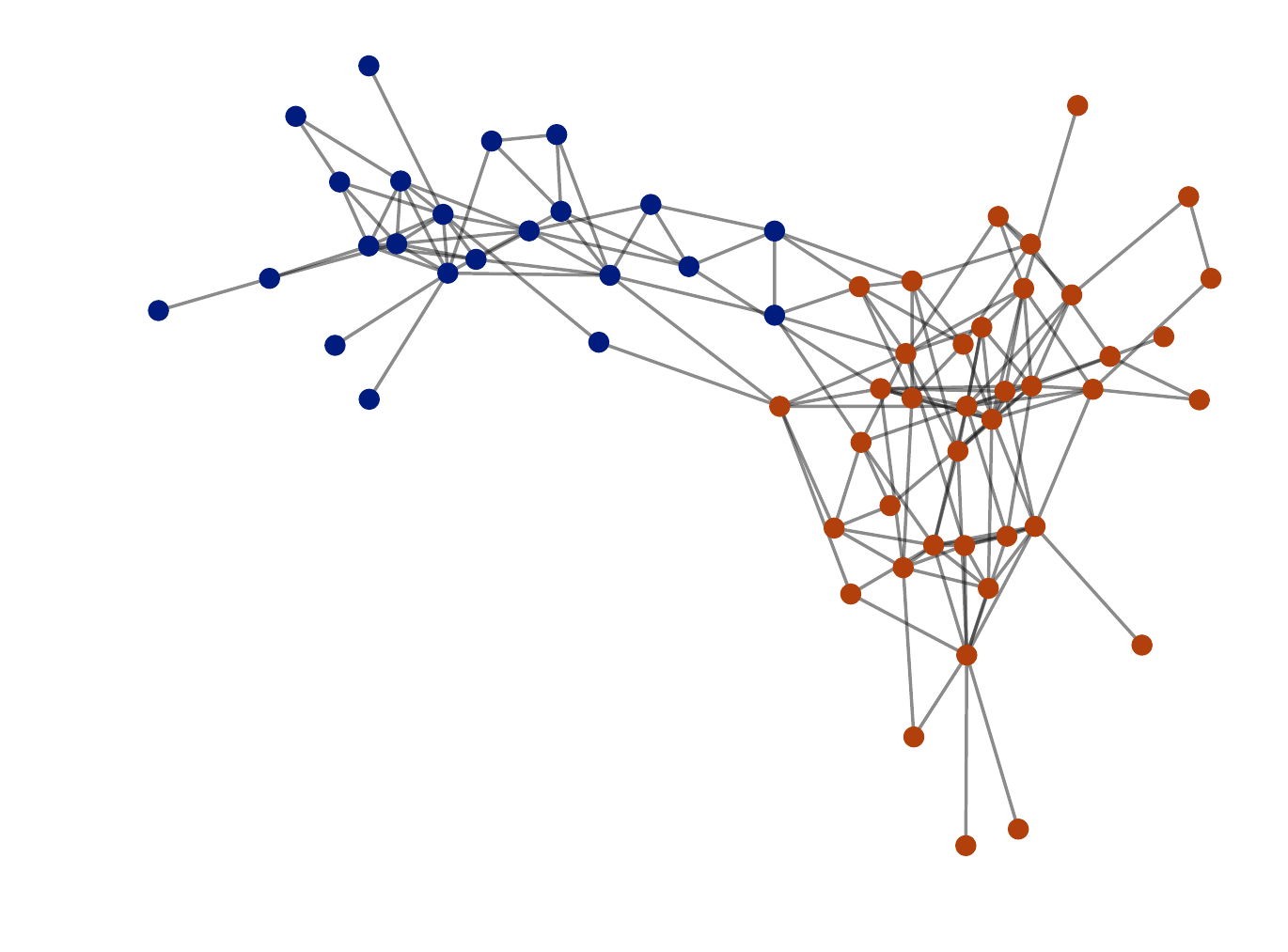}
\caption{Network}
\end{subfigure}%
\begin{subfigure}{.20\textwidth}
\centering
\includegraphics[width=\linewidth]{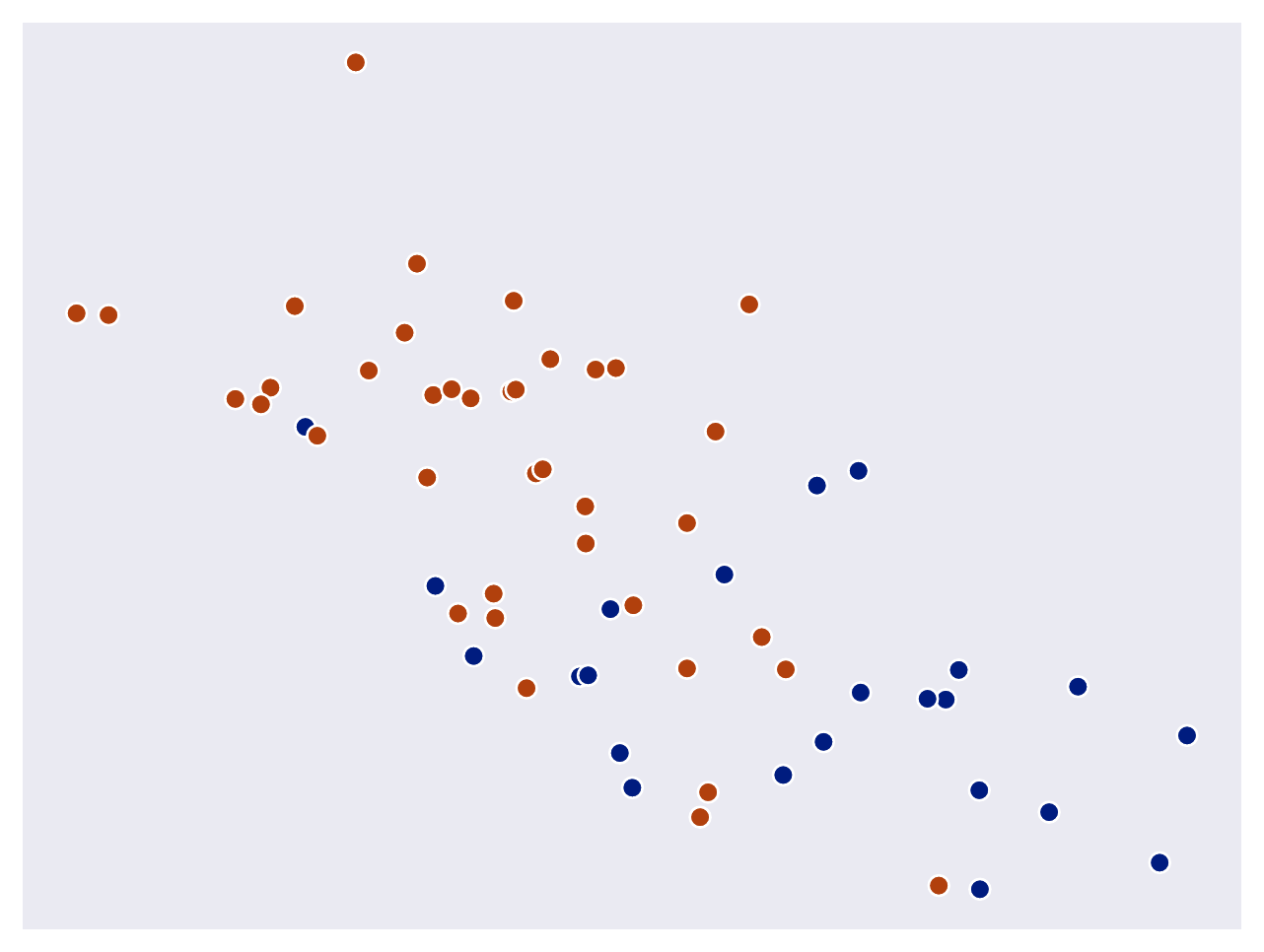}
\caption{\textsc{DeepWalk}}
\label{fig:sub1}
\end{subfigure}%
\begin{subfigure}{.20\textwidth}
\centering
\includegraphics[width=\linewidth]{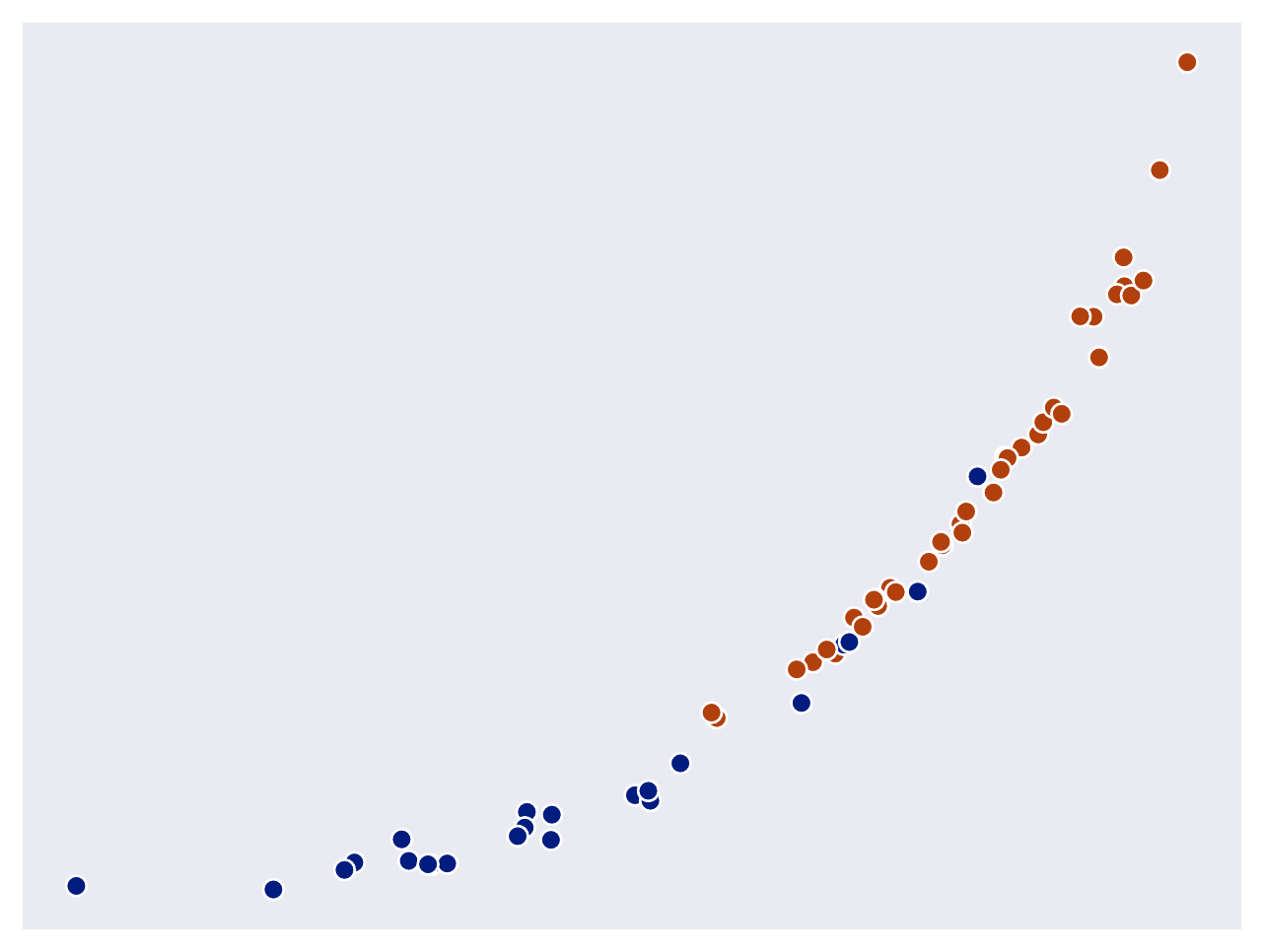}
\caption{\textsc{EFGE-Bern}}
\end{subfigure}%
\begin{subfigure}{.20\textwidth}
\centering
\includegraphics[width=\linewidth]{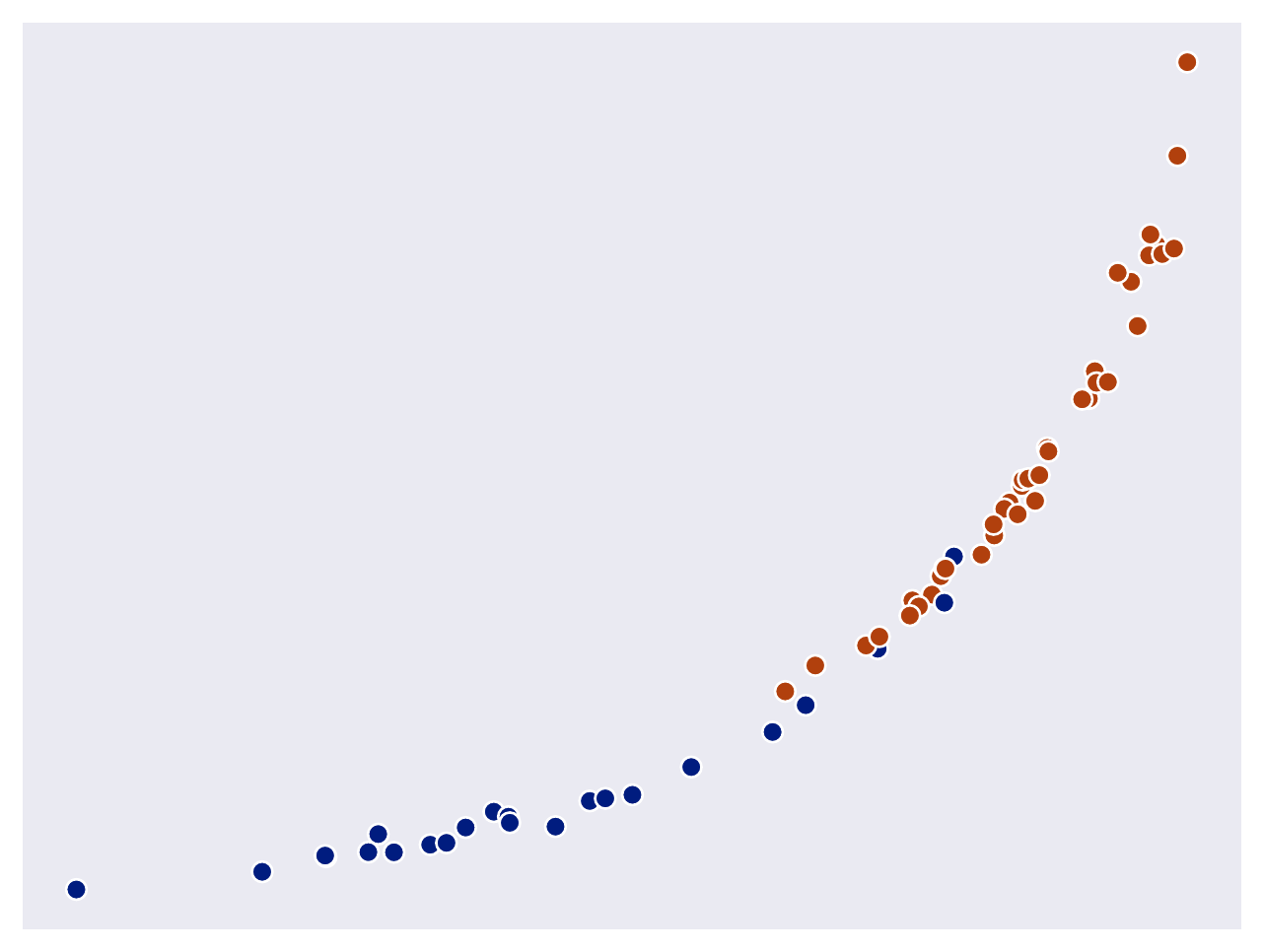}
\caption{\textsc{EFGE-Pois}}
\end{subfigure}%
\begin{subfigure}{.20\textwidth}
\centering
\includegraphics[width=\linewidth]{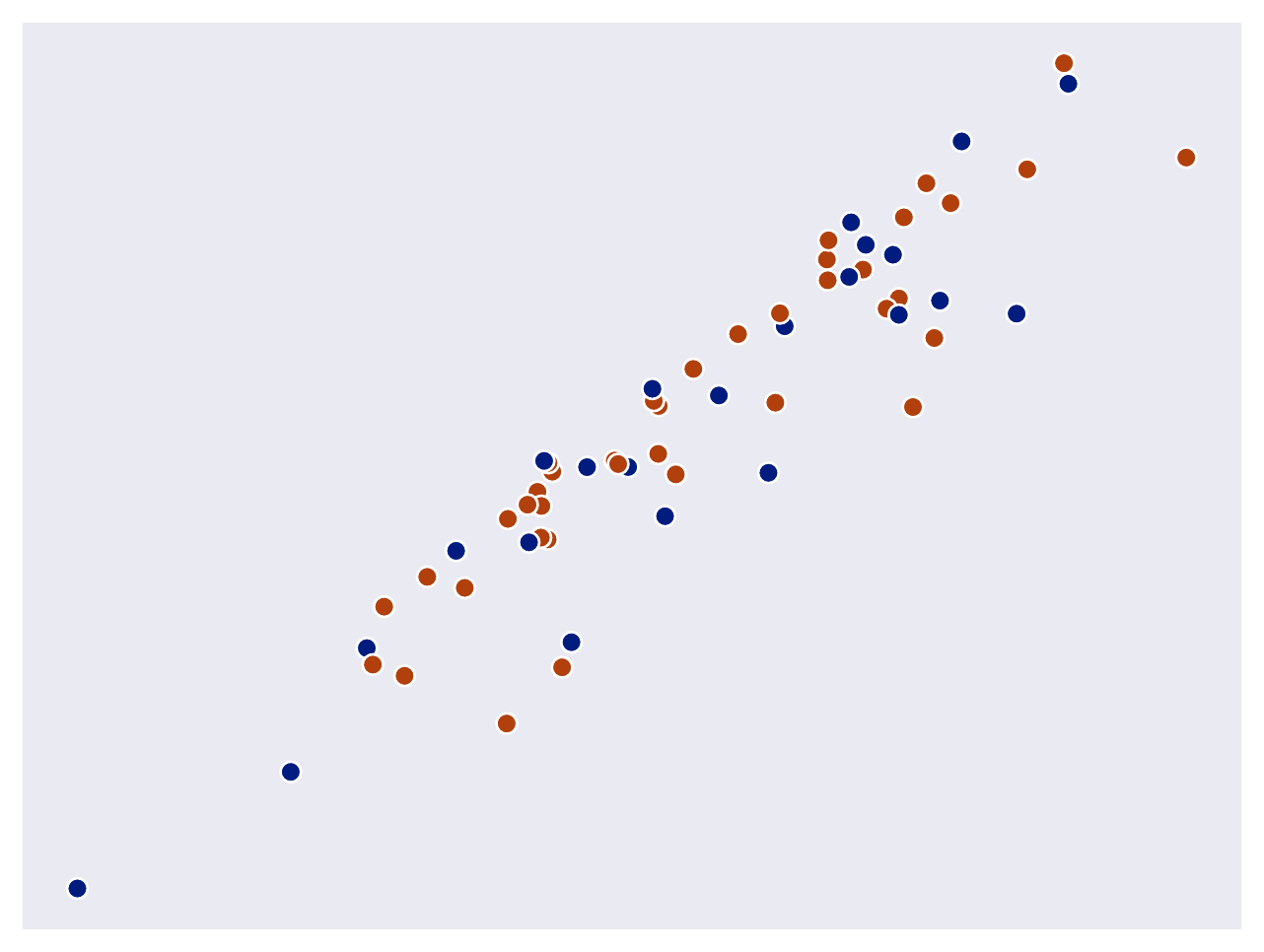}
\caption{\textsc{EFGE-Norm}}
\end{subfigure}
\caption{The \textsl{Dolphins} network composed by $2$ communities and the corresponding embeddings for $d=2$. \label{fig:dolphins}}
\end{figure*}

Initially, we sample a set of $N$ random walks based on a chosen walk strategy. This strategy can be any context sampling process, such as uniform random walks (as in \textsc{DeepWalk}) or biased random walks (as in \textsc{Node2Vec}). Then, based on the chosen instance of the \textsc{EFGE} model, we learn center and context embedding vectors. 
In this paper, we have examined three particular instances of the \textsc{EFGE} model, that represent well known exponential family distributions. In particular, we utilize the Bernoulli, Poisson, and Normal distributions leading to the corresponding \textsc{EFGE-Bern}, \textsc{EFGE-Pois} and \textsc{EFGE-Norm} models. For illustration purposes, Fig. \ref{fig:dolphins} depicts the \textsl{Dolphins} network composed by two communities and the embeddings in two dimensions as computed by different models. As we can observe, for this particular toy example, the proposed \textsc{EFGE-Bern} and \textsc{EFGE-Pois} models learn representations that are able to differentiate nodes with respect to their communities. In the following sections, we analyze the properties of these models in detail.

\subsection{The \textsc{EFGE-Bern} Model}

Our first model is the \textsc{EFGE-Bern} model, in which we assume that each $y_{w_l,v}$ follows a Bernoulli distribution which is equal to $1$ if node $v$ appears in the context set of $w_l$ in the walk $\textbf{w} \in \mathcal{W}$. It can be written as $y_{w_l,v}$ $=$ $x_{w_l,v}^{l-\gamma}$ $\vee$ $\cdots$ $\vee x_{w_l,v}^{l-1} \vee x_{w_l,v}^{l+1}$ $\vee$ $\cdots$ $\vee x_{w_l,v}^{l+\gamma}$, where $x_{w_l,v}^{l+j}$  indicates the appearance of $v$ in the context of $w_l$ at the specific position $j$ $(-\gamma\leq j \leq \gamma)$. We can express the objective function of the \textsc{EFGE-Bern} model, $\mathcal{L}_B(\mathbf{\alpha}, \mathbf{\beta})$, by dividing Eq. \eqref{eq:main_objective_exp} into two parts with respect to the values of $x_{w_l,v}^{l+j}$:

\begin{align*}
    \mathcal{L}_B =& \sum_{\textbf{w} \in \mathcal{W}} \sum_{1\leq l \leq L}\!\left[  \sum_{v \in \mathcal{N}_{\gamma}(w_l)}\!\!\!\!\! \log p(y_{w_l,v}) + \!\!\!\!\!\!\!\sum_{v \not\in \mathcal{N}_{\gamma}(w_l)}\!\!\!\!\!\log p(y_{w_l,v})\right]\\
    =& \sum_{\textbf{w} \in \mathcal{W}} \sum_{1\leq l \leq L}\!\left[  \sum_{\substack{| j | \leq \gamma \\ u^+:=w_j}}\!\!\!\!\!\log p(x_{w_l,u^+}^{l+j}) \! + \!\!\!\!\! \sum_{\substack{| j | \leq \gamma \\ u^-:\not=w_j}}\!\!\!\!\!\log p(x_{w_l,u^-}^{l+j}\!)\right]
\end{align*}

\noindent Note that, the exponential form of a Bernoulli distribution with a parameter $\pi$ is $\exp{( \eta x - A(\eta) )}$, where the log-normalizer $A(\eta)$ is $\log( 1 + \exp(\eta) )$ and the parameter $\pi$ is the sigmoid function $\sigma(\eta) = 1/ \big( 1 - \exp(-\eta) )\big)$. Therefore, we can rewrite the objective function $\mathcal{L}_B(\mathbf{\alpha}, \mathbf{\beta})$ as follows:

\begin{align*}
     \sum_{\textbf{w} \in \mathcal{W}} \sum_{1\leq l \leq L}\left[ \sum_{\substack{| j | \leq \gamma \\ u^+:=w_j}}\!\!\! \log\sigma(\eta_{w_l,u^+}) + \!\!\!\!\!\sum_{\substack{| j | \leq \gamma \\ u^-:\not=w_j}} \!\!\!\!\!\log\sigma(-\eta_{w_l,u^-})\right]
\end{align*}

\noindent We choose the identity map for the link function $f(\cdot)$, so $\eta_{v,u}$ becomes equal to the product of  vectors $\mathbf{\alpha}[v]$ and $\mathbf{\beta}[u]$. 

\vspace{.2cm}

\noindent \textbf{Relationship to negative sampling.} Although the \textit{negative sampling} strategy \cite{word2vec} was proposed to approximate the objective function of the \textit{Skip-Gram} model for node representation, any rigorous theoretical argument showing the connection between them has not been provided. In Lemma \ref{prop:negative_sampling}, we show that the log-likelihood $\mathcal{L}_B(\mathbf{\alpha}, \mathbf{\beta})$ of the  \textsc{EFGE-Bern} model in fact converges to the objective function of negative sampling given in Eq. \eqref{eq:negative_sampling}. In our implementation,  we adopt negative sampling in order to improve the efficiency of the computation.

\begin{lemma}\label{prop:negative_sampling}
The log-likelihood function $\mathcal{L}_B$ converges to 
\begin{align}\label{eq:negative_sampling}
\!\!\!\!\!\!\sum_{\textbf{w} \in \mathcal{W}} \sum_{1\leq l \leq L}\sum_{|j| \leq \gamma}\!\left[ \log p\big(x_{w_l,w_{l+j}}^{l+j}\big)\!\! +\!\! \sum_{s=1}^{k}\!\!\!\underset{\ \ u \sim q^-}{\!\!\mathbb{E}}\!\!\!\!\!\log p\big(x_{w_l,u}^{l+j}\big)\right]\!\!\!\!\!
\end{align}
for large values of $k$.
\end{lemma}
\begin{proof}
Please see the appendix.
\end{proof}

\subsection{The \textsc{EFGE-Pois} Model}
In this model, we will use the Poisson distribution to capture the relationship between context and center nodes in a random walk sequence. Let $y_{w_l,v}$ be a value indicating the number of occurrences of node $v$ in the context of $w_l$. We assume that $y_{w_l,v}$ follows a Poisson distribution, with the mean value $\tilde{\lambda}_{w_l,v}$ being the number of appearances of node $v$ in the context $\mathcal{N}_{\gamma}^{\textbf{w}}(w_l)$. Similar to the previous model, it can be expressed as $y_{w_l,v}$ = $x_{w_l,v}^{l-\gamma}$ + $\cdots$ + $x_{w_l,v}^{l-1}+x_{w_l,v}^{l+1}$ + $\cdots$ + $x_{w_l,v}^{l+\gamma}$, where $x_{w_l,v}^{l+j} \sim Pois(\lambda_{w_l,v})$ for $-\gamma \leq j \leq \gamma$. That way, we obtain $\tilde{\lambda}_{w_l,v} = \sum_{j=-\gamma}^{\gamma}\lambda_{w_l,v}^{l+j}$, since the sum of independent Poisson random variables is also Poisson.  By plugging the exponential form of the Poisson distribution into  Eq. \eqref{eq:main_objective}, we can derive the objective function $\mathcal{L}_P(\alpha,\beta)$ of the model as:

\begin{align*}
     \sum_{\textbf{w} \in \mathcal{W}} \sum_{1\leq l \leq L}\sum_{v \in \mathcal{V}}\left[\log h(y_{w_l,v})\!+\!\!\Big( \eta_{w_l,v} y_{w_l,v}\! - \exp(\eta_{w_l,v}) \Big)\right],
\end{align*}
where the base measure $h(y_{w_l,v})$ is equal to $1/ y_{w_l,v}!$. Note that, the number of occurrence $y_{w_l,v}$ is equal to $0$ if $v$ does not appear in the context of $w_l\in\mathcal{V}$. Following a similar strategy as in the \textsc{EFGE-Bern} model, the equation can be split into two parts for the cases where $y_{w_l,v} > 0$ and $y_{w_l,v}=0$. That way,  we can adopt the negative sampling strategy (given in Eq. \eqref{eq:negative_sampling}) as follows:

\begin{align*}
&\sum_{\textbf{w} \in \mathcal{W}} \sum_{1\leq l \leq L}\!\sum_{\substack{| j | \leq \gamma \\ u:=w_j}}\!\!\!\Big[\!\!-\!\!\log(x_{w_l,u}^{l+j}!)\! + \!\eta_{w_l,u} x_{w_l,u}^{l+j}\!-\! \exp(\eta_{w_l,u})\Big]\\
&\quad\quad\quad\quad + \!\!\! \sum_{\substack{| j | \leq \gamma \\ u:\not=w_j}}\!\!\!\Big[ -\exp(\eta_{w_l,u}) \Big].
\end{align*}

\noindent Note that, in the \textsc{EFGE-Pois} model, we do not specify any particular link function --- thus, the natural parameter is equal to the product of the embeddings vectors. 

\vspace{.2cm}

\noindent \textbf{Relationship to overlapping community detection.} It can be seen that the objective function of the widely used  \textsc{BigClam}  overlapping community detection method by Yang and Leskovec \cite{bigclam},  can be obtained by unifying  the objectives of the \textsc{EFGE-Bern} and \textsc{EFGE-Pois} models. The relationship is shown in Lemma \ref{label:label2}. Besides, one can say that each entry of the embedding vectors correspond to a  value indicating the membership of a node to a community --- in this case, \textsc{BigClam} restricts the vectors to non-negative values.

\begin{lemma}\label{label:label2}
Let $Z_{w_l,v}$ be independent random variables following Poisson distribution with natural parameter $\eta_{w_l,v}$ defined by $\log(\mathbf{\beta}[w_l]\cdot\mathbf{\alpha}[v])$. Then, the objective function of \textsc{EFGE-Bern} model becomes equal to
\begin{align*}
  \sum_{\textbf{w} \in \mathcal{W}} \sum_{1 \leq l \leq L} \Bigg[&\sum_{v \in \mathcal{N}_{\gamma}(w_l)}\!\!\!\!\!\log\Big(1 - \exp\big(-\mathbf{\beta}[w_l]^{\top}\cdot\mathbf{\alpha}[v]\big)\Big) \\
  -& \sum_{v \not\in \mathcal{N}_{\gamma}(w_l)} \!\!\!\!\!\mathbf{\beta}[w_l]^{\top}\cdot\mathbf{\alpha}[v]\Bigg]
\end{align*}
if the model parameter $\pi_{w_l,v}$ defined by $p(Z_{w_l,v}>0)$.
\end{lemma}
\begin{proof}
Please see the appendix.
\end{proof}

\subsection{The \textsc{EFGE-Norm} Model}

If a node $v$ appears in the context of $w_l$ more frequently with respect to other nodes, we can say that $v$ has a higher interaction with $w_l$ than the rest ones. Therefore, we will consider each $y_{w_l,v}$ in this model as an edge weight indicating the relationship between the nodes $w_l$ and $v$. We assume that $x_{w_l,v}^{l+j} \sim \mathcal{N}(1, \sigma^2_+)$ if $v \in \mathcal{N}_{\gamma}(w_l)$, and $x_{w_l,v}^{l+j} \sim \mathcal{N}(0, \sigma^2_-)$ otherwise. Hence, we obtain that $y_{w_l,v}\sim \mathcal{N}(\tilde{\mu},\tilde{\sigma^2})$,  where $\tilde{\mu}$ is the number of occurrences of $v$ in the context if we follow a similar assumption $y_{w_l,v}=\sum_{j=-\gamma}^{\gamma}x_{w_l, v}^{l+j}$ as in the previous model. The definition of the objective function, $\mathcal{L}_{N}(\mathbf{\alpha}, \mathbf{\beta})$, for the \textsc{EFGE-Norm} model is defined as follows:

\begin{align*}
    &\sum_{\textbf{w} \in \mathcal{W}}\sum_{1 \leq l \leq L}\sum_{\substack{| j | \leq \gamma \\ u:=w_j}}\left[ \log h(x_{w_l,u}^{l+j})\! +\! \Big(x_{w_l,u}^{l+j}\frac{\eta_{w_l,u}}{\sigma^+}\!-\!\frac{\eta_{w_l,u}^2}{2} \Big)\right] \nonumber \\
    &\quad\quad\quad\quad +\!\! \sum_{\substack{|j| \leq \gamma \\ u:\not=w_j}}\!\left[ \log h(x_{w_l,u}^{l+j}) +\! \Big(x_{w_l,u}^{l+j}\frac{\eta_{w_l,u}}{\sigma^-}\!-\!\frac{\eta_{w_l,u}^2}{2} \Big)\right],
\end{align*}

\noindent where the base measure $h(x_{w_l,u})$ is $\exp(-x_{w_l,u}^2 /  2\sigma^2)/ \sqrt{2\pi}\sigma$ for known variance. In this model, we choose the link function as  $f(x)= \exp(-x)$, so $\eta_{w_l,u}$ is defined as $\exp(-\beta[w_l]^\top\alpha[u])$.

\subsection{Optimization}\label{sec:optimization}
For the optimization 
we use  \textit{Stochastic Gradient Descent} (SGD) \cite{sgd} to learn representations $\Omega=(\alpha, \beta)$. Since we use exponential family distributions, we have a general form of the objective function given in Eq. \eqref{eq:main_objective_exp}. As it is computationally very expensive to compute gradients for each node pair, we take advantage of the fact that we have formulated the objective function of each model in such a way that it could be divided into two parts according to the values of $x_{w_l,u}^{l+j}$; thus, we adopt the negative sampling strategy, setting sampling size to $k=5$ in all the experiments. For the update of learning parameters and for generating negative samples, we follow the approach described in \cite{deepwalk,word2vec}.

\section{Experimental Evaluation} \label{sec:experiments}
In this section, we evaluate the performance of the proposed models with respect to several node embedding baseline techniques in the node classification and link prediction tasks over various networks shown in Table \ref{tab:networks}. 

\subsubsection*{Baseline methods.} We evaluate the three proposed \textsc{EFGE} models against five state-of-the-art NRL techniques.
$(i)$ \textsc{DeepWalk} \cite{deepwalk} generates a set of node sequences by choosing a node uniformly at random from the neighbours of the node it currently resides.
$(ii)$ \textsc{Node2Vec} \cite{node2vec} relies on a biased random walk strategy, introducing two additional parameters which are used to determine the behaviour of the random walk in visiting nodes close to the one currently residing at. We simply set these parameters to $1.0$. $(iii)$ \textsc{LINE} \cite{line} learns embeddings that are based on first-order and second-order proximity (each one of length $d/2$).
$(iv)$ \textsc{HOPE} \cite{hope} is a matrix factorization method which aims at extracting feature vectors by preserving  higher order patterns of the network (in our experiments, we have used the \textit{Katz} index).
$(v)$ \textsc{NetMF} \cite{implicit_factorization} aims at factorizing the matrix approximated by the pointwise mutual information of center and context pairs. In our experiments, we have used walk length $L=10$, number of walks $N=80$ and window size $\gamma=10$ for all models and the variants of \textsc{EFGE} model are fed with the same node sequences produced by \textsc{Node2Vec}.

\begin{table}[t]
\centering
\resizebox{1.0\columnwidth}{!}{
\begin{tabular}{r|cccccc}
\toprule
\multicolumn{1}{l}{} & $|\mathcal{V}|$ & $|\mathcal{E}|$ & $|\mathcal{K}|$ & $|\mathcal{C}|$ & \textbf{Avg. Degree} & \textbf{Density} \\\midrule
\textsl{CiteSeer} & 3,312 & 4,660 & 6 & 438 & 2.814 & 0.0009 \\
\textsl{Cora} & 2,708 & 5,278 & 7 & 78 & 3.898 & 0.0014 \\
\textsl{DBLP} & 27,199 & 66,832 & 4 & 2,115 & 4.914 & 0.0002 \\
\textsl{AstroPh} & 17,903 & 19,7031 & - & 1 & 22.010 & 0.0012 \\
\textsl{HepTh} & 8,638 & 24,827 & - & 1 & 5.7483 & 0.0007 \\
\textsl{Facebook} & 4,039 & 88,234 & - & 1 & 43.6910 & 0.0108 \\
\textsl{GrQc} & 4,158 & 13,428 & - & 1 & 6.4589 & 0.0016 \\\bottomrule
\end{tabular}
}
\caption{Statistics of network datasets used in the experiments. $|\mathcal{V}|$: number of nodes, $|\mathcal{E}|$: number of edges, $|\mathcal{K}|$: number of labels and $|\mathcal{C}|$: number of connected components.}
\label{tab:networks}
\end{table}

	
	
	
	
	
	
	


\subsection{Node Classification}
\noindent \textbf{Experimental setup.} In the classification task, we aim to predict the correct labels of nodes having access to a limited number of training labels (i.e., nodes with known label). In our experiments, we split the nodes into varying training ratios, from $2\%$ up to $90\%$ in order to better evaluate the models. We perform our experiments applying an one-vs-rest logistic regression classifier with $L_2$ regularization\footnote{We have used the \textit{scikit-learn} package in the implementation.}, computing the Micro-$F_1$ score (the Macro-$F_1$ score over a wide range of training ratios is also presented in the Appendix). We repeat the experiments for $50$ times and report the average score for each network.

\begin{table}[h]
\begin{subtable}[h]{\columnwidth}
\centering
\resizebox{\columnwidth}{!}{%
\begin{tabular}{r|ccccccccc}
 \multicolumn{1}{c}{} & \textbf{2\%} & \textbf{4\%} & \textbf{6\%} & \textbf{8\%} & \textbf{10\%} & \textbf{30\%} & \textbf{50\%} & \textbf{70\%} & \textbf{90\%} \\\midrule
\multirow{1}{*}{\textsc{DeepWalk}} & 0.416 & 0.460 & 0.489 &  0.505 & 0.517 & 0.566 & 0.584 & 0.595 & 0.592 \\
 \multirow{1}{*}{\textsc{Node2Vec}} & \cellcolor[HTML]{F2F3F4}0.450 &  \cellcolor[HTML]{F2F3F4}0.491 & \cellcolor[HTML]{F2F3F4}0.517 & \cellcolor[HTML]{F2F3F4}0.530 &  \cellcolor[HTML]{F2F3F4}0.541 &  \cellcolor[HTML]{F2F3F4}0.585 &  \cellcolor[HTML]{F2F3F4}0.597 &  \cellcolor[HTML]{F2F3F4}0.601 &  \cellcolor[HTML]{F2F3F4}0.599 \\
 \multirow{1}{*}{\textsc{LINE}} & 0.323 & 0.387 & 0.423 & 0.451 & 0.466 & 0.532 & 0.551 & 0.560 & 0.564 \\
 \multirow{1}{*}{\textsc{HOPE}} & \cellcolor[HTML]{F2F3F4}0.196 & \cellcolor[HTML]{F2F3F4}0.205 & \cellcolor[HTML]{F2F3F4}0.210 & \cellcolor[HTML]{F2F3F4}0.204 & \cellcolor[HTML]{F2F3F4}0.219 & \cellcolor[HTML]{F2F3F4}0.256 & \cellcolor[HTML]{F2F3F4}0.277 & \cellcolor[HTML]{F2F3F4}0.299  & \cellcolor[HTML]{F2F3F4}0.320
 \\
 \multirow{1}{*}{\textsc{NetMF}} & 0.451 & 0.496 & 0.526 & 0.540 & 0.552 & 0.590 & 0.603 & 0.604 & 0.608
 \\\midrule
\multirow{1}{*}{\textsc{EFGE-Bern}} & 0.461 & 0.493 & 0.517 & 0.536 & 0.549 & 0.588 & 0.603 & 0.609 & 0.609 \\
\multirow{1}{*}{\textsc{EFGE-Pois}} & \cellcolor[HTML]{F2F3F4}0.484  & \cellcolor[HTML]{F2F3F4} 0.514 & \cellcolor[HTML]{F2F3F4} 0.537 & \cellcolor[HTML]{F2F3F4}0.551 & \cellcolor[HTML]{F2F3F4}\textbf{0.562} & \cellcolor[HTML]{F2F3F4}0.595 & \cellcolor[HTML]{F2F3F4}\textbf{0.606} & \cellcolor[HTML]{F2F3F4}0.611 & \cellcolor[HTML]{F2F3F4}0.613 \\
\multirow{1}{*}{\textsc{EFGE-Norm}} & \multicolumn{1}{c}{\textbf{0.493}} & \multicolumn{1}{c}{\textbf{0.525}} & \multicolumn{1}{c}{\textbf{0.542}} & \multicolumn{1}{c}{\textbf{0.553}} & \multicolumn{1}{c}{0.561} & \multicolumn{1}{c}{\textbf{0.596}} & \multicolumn{1}{c}{\textbf{0.606}} &  \multicolumn{1}{c}{\textbf{0.612}} &  \multicolumn{1}{c}{\textbf{0.616}} \\\bottomrule
\end{tabular}%
}
\caption{\textsl{CiteSeer}}
\label{tab:classification_citeseer}
\end{subtable}
\vfill
\begin{subtable}[h]{\columnwidth}
\centering
\resizebox{\columnwidth}{!}{%
\begin{tabular}{r|ccccccccc}
\multicolumn{1}{l}{} &\textbf{2\%} & \textbf{4\%} & \textbf{6\%} & \textbf{8\%} & \textbf{10\%} & \textbf{30\%} & \textbf{50\%} & \textbf{70\%} & \textbf{90\%} \\\midrule
\multirow{1}{*}{\textsc{DeepWalk}} & 0.621 & 0.689 & 0.715 & 0.732 & 0.747 & 0.802 & 0.819 & 0.826 & 0.833 \\
 \multirow{1}{*}{\textsc{Node2Vec}} &\cellcolor[HTML]{F2F3F4}0.656 & \cellcolor[HTML]{F2F3F4}0.714 & \cellcolor[HTML]{F2F3F4}0.743 & \cellcolor[HTML]{F2F3F4}0.757 & \cellcolor[HTML]{F2F3F4}0.769 & \cellcolor[HTML]{F2F3F4}0.815 & \cellcolor[HTML]{F2F3F4}0.831 & \cellcolor[HTML]{F2F3F4}0.839 & \cellcolor[HTML]{F2F3F4}0.841 \\
 \multirow{1}{*}{\textsc{LINE}} & 0.450 & 0.544 & 0.590 & 0.633 & 0.661 & 0.746 & 0.765 & 0.774 & 0.775 \\
 \multirow{1}{*}{\textsc{HOPE}} & \cellcolor[HTML]{F2F3F4}0.277 &\cellcolor[HTML]{F2F3F4}0.302 & \cellcolor[HTML]{F2F3F4}0.299 & \cellcolor[HTML]{F2F3F4}0.302 & \cellcolor[HTML]{F2F3F4}0.302 & \cellcolor[HTML]{F2F3F4}0.301 &\cellcolor[HTML]{F2F3F4}0.302 & \cellcolor[HTML]{F2F3F4}0.303 & \cellcolor[HTML]{F2F3F4}0.302 \\
 \textsc{NetMF} & 0.636 & 0.716 & 0.748 & 0.767 & 0.773 & \textbf{0.821} & \textbf{0.834} & \textbf{0.841} & \textbf{0.844}
 \\\midrule
\multirow{1}{*}{\textsc{EFGE-Bern}} & 0.668 & 0.720 & 0.743 & 0.759 & 0.767 & 0.808 & 0.823 & 0.834 & 0.838 \\
 \multirow{1}{*}{\textsc{EFGE-Pois}} & \cellcolor[HTML]{F2F3F4}0.680 & \cellcolor[HTML]{F2F3F4} 0.733 & \cellcolor[HTML]{F2F3F4}0.746 & \cellcolor[HTML]{F2F3F4}0.759 & \cellcolor[HTML]{F2F3F4}0.765 & \cellcolor[HTML]{F2F3F4}0.802 & \cellcolor[HTML]{F2F3F4}0.814 &  \cellcolor[HTML]{F2F3F4}0.820 & \cellcolor[HTML]{F2F3F4}0.825 \\
 \multirow{1}{*}{\textsc{EFGE-Norm}} & \multicolumn{1}{c}{\textbf{0.682}} & \multicolumn{1}{c}{\textbf{0.743}} & \multicolumn{1}{c}{\textbf{0.760}} & \multicolumn{1}{c}{\textbf{0.770}} & \multicolumn{1}{c}{\textbf{0.780}} & \multicolumn{1}{c}{0.810} & \multicolumn{1}{l}{0.824} & \multicolumn{1}{c}{0.827} & \multicolumn{1}{l}{0.839} \\\bottomrule
\end{tabular}%
}
\caption{\textsl{Cora}}
\label{tab:classification_cora}
\end{subtable}
\vfill
\begin{subtable}[h]{\columnwidth}
\centering
\resizebox{\columnwidth}{!}{%
\begin{tabular}{r|ccccccccc}
 & \textbf{2\%} & \textbf{4\%} & \textbf{6\%} & \textbf{8\%} & \textbf{10\%} & \textbf{30\%} & \textbf{50\%} & \textbf{70\%} & \textbf{90\%} \\\midrule
\multirow{1}{*}{\textsc{DeepWalk}} & 0.545 & 0.585 & 0.600 & 0.608 & 0.613 & 0.626 & 0.628 & 0.628 & 0.633 \\
 \multirow{1}{*}{\textsc{Node2Vec}} &\cellcolor[HTML]{F2F3F4}0.575 & \cellcolor[HTML]{F2F3F4}0.600 & \cellcolor[HTML]{F2F3F4}0.611 & \cellcolor[HTML]{F2F3F4}0.619 & \cellcolor[HTML]{F2F3F4}0.622 & \cellcolor[HTML]{F2F3F4}0.636 &
 \cellcolor[HTML]{F2F3F4}0.638 &
 \cellcolor[HTML]{F2F3F4}0.639 & \cellcolor[HTML]{F2F3F4}0.639 \\
 \multirow{1}{*}{\textsc{LINE}} & 0.554 & 0.580 & 0.590 & 0.597 & 0.603 & 0.618 & 0.621 & 0.623 & 0.623 \\
 \multirow{1}{*}{\textsc{HOPE}} & \cellcolor[HTML]{F2F3F4}0.379 & \cellcolor[HTML]{F2F3F4}0.378 & \cellcolor[HTML]{F2F3F4}0.379 & \cellcolor[HTML]{F2F3F4}0.379 & \cellcolor[HTML]{F2F3F4}0.379 & \cellcolor[HTML]{F2F3F4}0.379 & \cellcolor[HTML]{F2F3F4}0.379 & \cellcolor[HTML]{F2F3F4}0.378 & \cellcolor[HTML]{F2F3F4}0.380 \\
 \multirow{1}{*}{\textsc{NetMF}} & 0.577 & 0.589 & 0.596 & 0.601 & 0.605 & 0.617 & 0.620 & 0.623 & 0.623
 \\\midrule
\multirow{1}{*}{\textsc{EFGE-Bern}} & 0.573 & 0.598 & 0.610 & 0.617 & 0.622 & 0.634 & 0.638 & 0.638 & 0.638 \\
 \multirow{1}{*}{\textsc{EFGE-Pois}} & \cellcolor[HTML]{F2F3F4}0.588 & \cellcolor[HTML]{F2F3F4}0.605 & \cellcolor[HTML]{F2F3F4}0.614 &
 \cellcolor[HTML]{F2F3F4}0.620 & \cellcolor[HTML]{F2F3F4}0.624 & \cellcolor[HTML]{F2F3F4}0.635 & \cellcolor[HTML]{F2F3F4}0.637 & \cellcolor[HTML]{F2F3F4}0.636 & \cellcolor[HTML]{F2F3F4}0.638 \\
 \multirow{1}{*}{\textsc{EFGE-Norm}} & \textbf{0.603} & \textbf{0.614} & \textbf{0.622} & \textbf{0.624} & \textbf{0.628} & \textbf{0.637} & \textbf{0.640} & \textbf{0.642} &  \textbf{0.641} \\\bottomrule
\end{tabular}%
}
\caption{\textsl{DBLP}}
\label{tab:classification_dblp}
\end{subtable}
\caption{Micro-$F_1$ scores for the node classification experiment for varying training sizes of networks.}
\label{tab:classification_table}
\end{table}

\subsubsection*{Experiment results.}

Table \ref{tab:classification_citeseer} shows the classification performance on the \textsl{CiteSeer} network. In all cases, the proposed models outperform the baselines, with the \textsc{EFGE-Norm} and \textsc{EFGE-Pois} models being the best performing ones. The \textsc{EFGE-Norm} model shows the best performance  among the three \textsc{EFGE} models, for most training sizes. The percentage gain for Micro-$F_1$ score of our best model with respect to the highest baseline score, is varying from $0.61\%$ up to $9.33\%$. For the results on the \textsl{Cora} network shown in  Table \ref{tab:classification_cora}, the \textsc{EFGE-Norm} model outperforms the baseline methods for small training set sizes of up to $10\%$. The \textsc{EFGE-Pois} also shows similar characteristics, while the \textsc{EFGE-Bern} model has comparable performance to \textsc{Node2Vec}.  The \textsc{EFGE-Norm} model has a gain of $4.0\%$  against the best of baselines. For large training sets above $30\%$, \textsc{NetMF} is the best performing model over the \textsl{Cora} network. Lastly, moving on the results on the \textsl{DBLP} network  shown in Table \ref{tab:classification_dblp}, the \textsc{EFGE-Norm} model shows the best performance in all cases under the Micro-$F_1$ scores. The highest Micro-$F_1$ gain of our proposed models against the best performing baseline is around $4.51\%$. 


Overall, the classification experiments show that the proposed \textsc{EFGE-Pois} and \textsc{EFGE-Norm} models perform quite well, outperforming most baselines especially on a limited number of training data. This can qualitatively be explained by the fact that, those exponential family distribution models enable to capture the number of occurrences of a node within the context of another one, while learning the embedding vectors. Of course, the structural properties of the network, such as the existence of community structure, might affect the performance of these models. For instance, as we have seen in the toy example of Fig. \ref{fig:dolphins}, the existence of well defined communities at the \textsl{Dolphins} network, allows the \textsc{EFGE-Pois} model to learn more discriminative embeddings with respect to the underlying communities (as we expect to have repetitions of nodes that belong to the same community while sampling the context of a node based on random walks).


\subsection{Link Prediction}
\noindent \textbf{Experimental set-up.} In the link prediction task, the goal is to predict the missing edges or to estimate possible future connections between nodes. For this experiment, we randomly remove half of the edges of a given network, keeping the residual network connected. Then, we learn node representations  using the residual network. The removed edges as well as a randomly chosen set of the same number of node pairs form the testing set. For the training set, we sample the same number of non-existing edges following the same strategy to have negative samples, and the edges in the residual network are used as positive instances. Since we learn embedding vectors for the nodes of the graph, we use the extracted node representations to build edge feature vectors using the \textit{Hadamard} product operator. Let $a, b \in \mathbb{R}^d$ be the embeddings of two nodes $u, v \in \mathcal{V}$ respectively. Then, under the Hadamard operator, the embedding of the corresponding edge between $u$ and $v$ will be computed as: $[a_1 * b_1, \cdots, a_d * b_d]$. In all experiments, we have used the logistic regression classifier with $L_2$ regularization over the networks listed in Table \ref{tab:networks}.

\subsubsection*{\normalfont \textbf{Experiment results.}} 

\begin{table}[t]
\resizebox{\linewidth}{!}{%
\begin{tabular}{r|ccccc|ccc}
 \multicolumn{1}{l}{}& \rotatebox{75}{\textsc{DeepWalk}} & \rotatebox{75}{\textsc{Node2Vec}} & \rotatebox{75}{\textsc{LINE}} & \rotatebox{75}{\textsc{HOPE}} &
 \rotatebox{75}{\textsc{NetMF}} &
 \rotatebox{75}{\textsc{EFGE-Bern}} & \rotatebox{75}{\textsc{EFGE-Pois}} & {\rotatebox{75}{\textsc{EFGE-Norm}}} \\\midrule
\textsl{Citeseer} & 0.770 & 0.780 & 0.717 & 0.744 & 0.742 & 0.815 & \textbf{0.834} & 0.828 \\
\textsl{Cora} & \cellcolor[HTML]{F2F3F4}0.739 & \cellcolor[HTML]{F2F3F4}0.757 & \cellcolor[HTML]{F2F3F4}0.686 & \cellcolor[HTML]{F2F3F4}0.712 & \cellcolor[HTML]{F2F3F4}0.755 & \cellcolor[HTML]{F2F3F4}0.769 & \cellcolor[HTML]{F2F3F4}0.797 & \cellcolor[HTML]{F2F3F4}\textbf{0.807} \\
\textsl{DBLP} & 0.919 & 0.954 & 0.933 & 0.873 & 0.930 & 0.950 & 0.950 & \textbf{0.955} \\
\textsl{AstroPh} & \cellcolor[HTML]{F2F3F4}0.911 & \cellcolor[HTML]{F2F3F4}0.969 & \cellcolor[HTML]{F2F3F4}0.971 & \cellcolor[HTML]{F2F3F4}0.931 & \cellcolor[HTML]{F2F3F4}0.897 & \cellcolor[HTML]{F2F3F4}0.963 & \cellcolor[HTML]{F2F3F4}0.922 & \cellcolor[HTML]{F2F3F4}\textbf{0.973} \\
\textsl{HepTh} & 0.843 & 0.896 & 0.854 & 0.836 & 0.882 & \textbf{0.898} & 0.885 & 0.896 \\
\textsl{Facebook} & \cellcolor[HTML]{F2F3F4}0.980 & \cellcolor[HTML]{F2F3F4}\textbf{0.992} & \cellcolor[HTML]{F2F3F4}0.986 & \cellcolor[HTML]{F2F3F4}0.975 & \cellcolor[HTML]{F2F3F4}0.987 & \cellcolor[HTML]{F2F3F4}0.991 & \cellcolor[HTML]{F2F3F4}0.991 & \cellcolor[HTML]{F2F3F4}\textbf{0.992} \\
\textsl{GrQc} & 0.921 & \textbf{0.940} & 0.909 & 0.902 & 0.928 & 0.938 & 0.937 & \textbf{0.940}\\\bottomrule
\end{tabular}%
}
\caption{Area Under Curve (AUC) scores for link prediction.}
\label{tab:link_prediction}
\end{table}

Table \ref{tab:link_prediction} shows the area under curve (AUC) scores for the link prediction task. Since the networks used in the node classification experiments consist of disconnected components, we perform the link prediction experiments on the largest connected component. As it can be seen in Table \ref{tab:link_prediction}, the \textsc{EFGE-Norm} model is performing quite well on almost all different types of networks. Although \textsc{Node2Vec} is quite effective  having similar performance in two datasets,  it is outperformed by \textsc{EFGE-Norm} from $0.04\%$ up to $18.29\%$ in the remaining networks.

\subsection{Parameter Sensitivity}
In this subsection, we evaluate how the performance of our models is affected under different parameter settings. In particular, we mainly examine the effect of embedding dimension $d$ and the effect of the window size $\gamma$ used to sample context nodes. More detailed analysis including the effect of the standard deviation $\sigma$ of the \textsc{EFGE-Norm} model is provided in the Appendix.

\subsubsection*{\normalfont \textbf{The effect of dimension size.}}


The dimension size $d$ of embedding vectors is a crucial parameter that can  affect the performance of a model. We have conducted experiments examining the effect of embedding dimension $d$  on the \textsl{Citeseer} network. As it can be seen in Fig. \ref{fig:subim1}, the increase in the dimension size has positive affect for all models over Micro-$F_1$ scores. When the dimension size increases from $32$ up to $224$, we observe a gain of around $18\%$  for  training set constructed from $50\%$ of the network.

\subsubsection*{The effect of window size.}


\begin{figure}[h]
\centering
\begin{subfigure}{0.49\columnwidth}
\includegraphics[width=0.92\columnwidth]{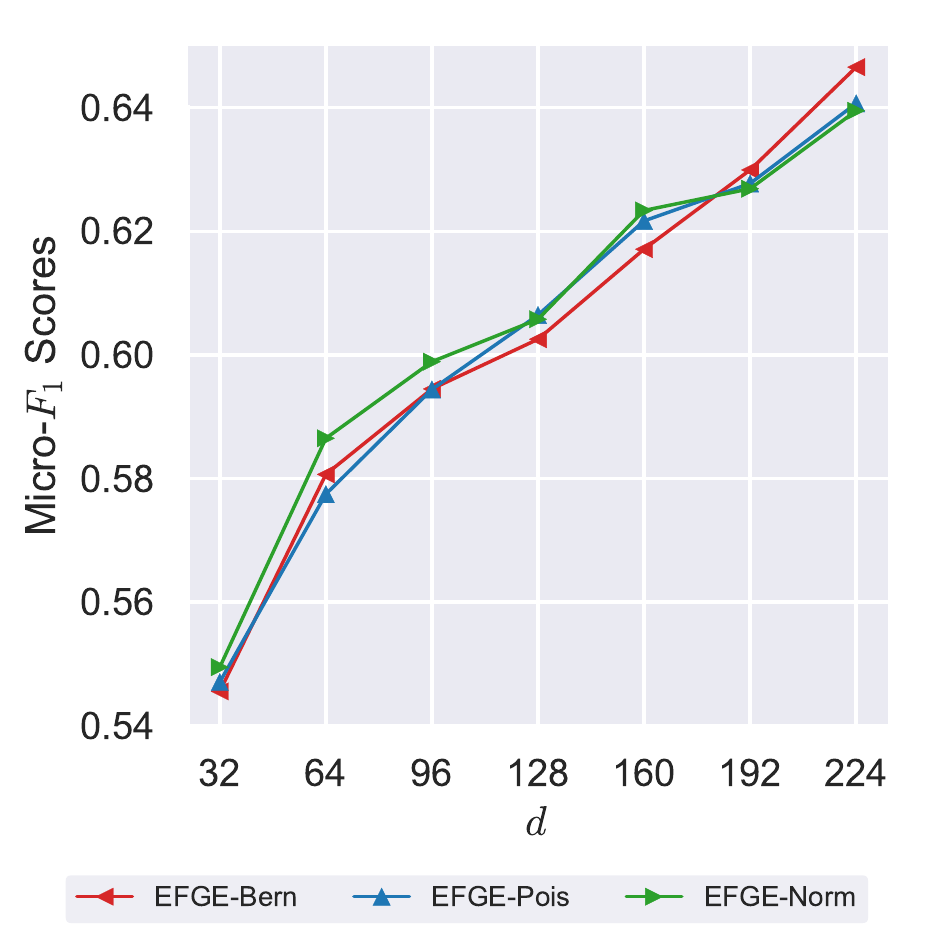}
\caption{Dimension of embeddings}
\label{fig:subim1}
\end{subfigure}
\begin{subfigure}{0.49\columnwidth}
\includegraphics[width=0.92\columnwidth]{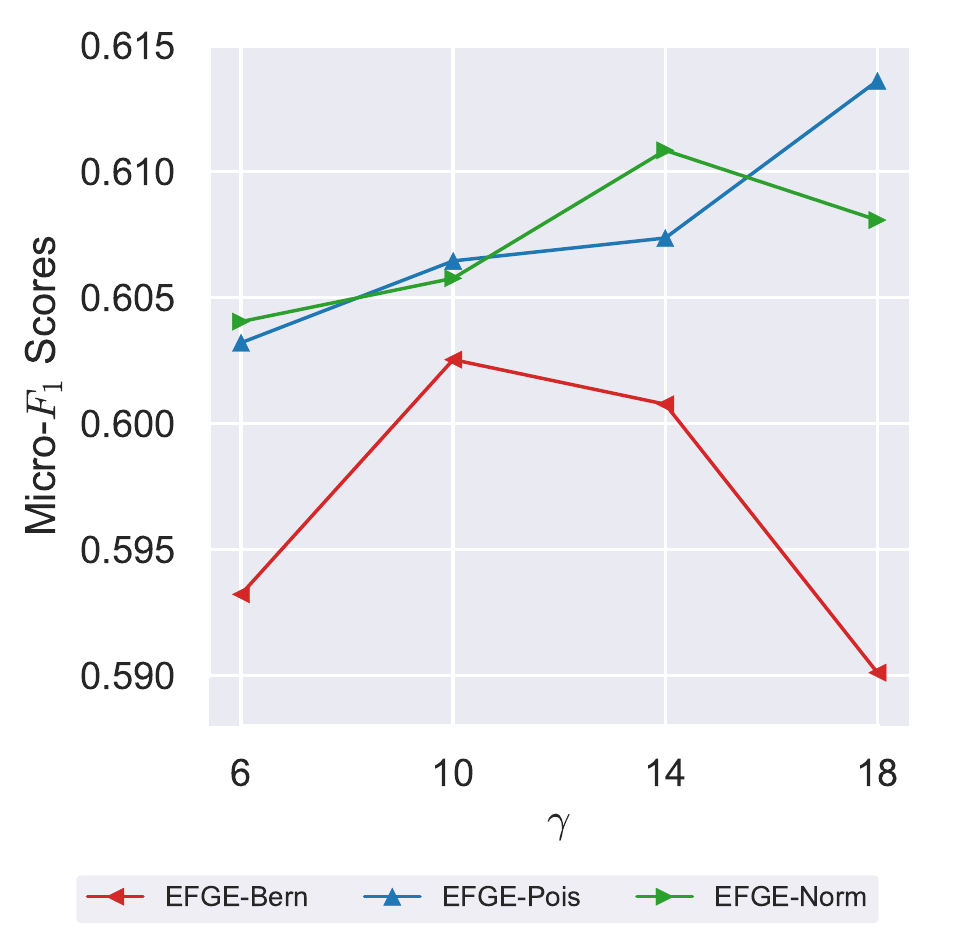} 
\caption{Window size}
\label{fig:subim2}
\end{subfigure}
\caption{Influence of dimension $d$ and  window size $\gamma$  on the \textsl{CiteSeer} network for the training set ratio of $50\%$.}
\label{fig:image2}
\end{figure}

Since the appearance or the number of occurrences of a node in the context of a center node is of importance for the \textsc{EFGE} models, we analyze their sensitivity under different window sizes $\gamma$ on the \textsl{CiteSeer} network. Figure \ref{fig:subim2} depicts the Micro-$F_1$ scores for training set composed by $50\%$  of the network. As we can observe, both the \textsc{EFGE-Norm} and  \textsc{EFGE-Pois} models have the tendency to show better performance for large window sizes, since they directly model the number of occurrences of nodes within a random walk sequence --- and potentially are benefited by a large $\gamma$ value. On the contrary, the performance of the \textsc{EFGE-Bern} model (which in fact captures simple co-occurrence relationships, resembling \textsc{Node2Vec}) deteriorates for large window sizes.




\section{Related Work}\label{sec:related}

\textbf{Network representation learning.} The traditional unsupervised feature learning methods aim at factorizing some matrix representation of the graph, which has been designed by taking into account  the properties  of a given network \cite{survey_hamilton_rex_leskovec}. Laplacian Eigenmaps \cite{laplacian_eigenmap} and IsoMap \cite{isomap} are just some of those approaches targeting to preserve  first-order proximity of nodes. More recently,  proposed algorithms including \textsc{GraRep} \cite{grarep} and HOPE \cite{hope}, aim at preserving higher order proximities. Nevertheless, despite the fact that matrix factorization approaches offer an elegant way to capture the desired properties, they mainly suffer from their time complexity.  \textsc{LINE} \cite{line} and \textsc{SDNE} \cite{wang2016structural} both optimize more sophisticated objective functions that preserve both first- and second-order proximities at the cost of an increased computational complexity, while \textsc{Verse} \cite{verse-www18} utilizes node similarity measures to learn node representations.
In addition, community structure properties can also be taken into account in the NRL process. The authors of \cite{DBLP:conf/aaai/WangCWP0Y17},   proposed a matrix factorization algorithm that incorporates the community structure into the embedding process, implicitly focusing on the quantity of modularity.


Random walk-based methods \cite{survey_hamilton_rex_leskovec} have gained considerable attention, mainly due the efficiency of the \textit{Skip-Gram} model. \textsc{DeepWalk} performs uniform random walks to sample context nodes, while \textsc{Node2Vec} and its extensions \cite{node2vec,biasedwalk} simulate biased-random walks that provide a trade-off between breadth-first and depth-first graph traversals. Following this line of research, distinct random sampling strategies have been proposed and various methods have emerged \cite{random_walk_struc2vec}. In all those cases though, the \textit{softmax} function is used model center-context relationships, something that might restrict the performance of the models. More recently, \textit{Skip-Gram}-based methods were extended to multiple vector representations, aiming at capturing multiple roles of nodes in the case of inherent overlapping communities \cite{epasto-splitter}. In addition, it was recently shown that   \textsc{DeepWalk} and \textsc{Node2Vec}  implicitly perform matrix factorizations \cite{implicit_factorization,netsmf-www2019}. 

\par Recently, there is an intense research effort on Graph Neural Network (GNN) architectures \cite{gnn}, including graph convolutional networks, autoencoders and diffusion models. Most of these approaches are supervised or semi-supervised, requiring labeled data in the training step, while here we are interested in unsupervised models. 

\vspace{.2cm}

\noindent \textbf{Exponential families.} In the related literature,  exponential family distributions have been utilized to learn embeddings for high-dimensional data of different types (e.g., market basket analysis) \cite{expon_fam_emb,NIPS2017_7067,NIPS2017_6629}. As we have presented, our approach generalizes exponential family embedding models to graphs.

\section{Conclusions} \label{sec:conclusions}
In this paper, we  introduced exponential family graph embeddings (\textsc{EFGE}), proposing three instances  (\textsc{EFGE-Bern}, \textsc{EFGE-Pois} and \textsc{EFGE-Norm}) that generalize random walk approaches to exponential families.  The benefit of these models stems from the fact that they allow to utilize exponential family distributions over center-context node pairs, going beyond simple co-occurrence relationships. We have also examined how the objective functions of the models can be expressed in a way that negative sampling can be applied to scale the learning process. The experimental results have demonstrated that instances of the EFGE model are able to outperform widely used baseline methods. As future work, we plan to further generalize the model to other exponential family distributions.

\bibliography{main}
\bibliographystyle{aaai}

\appendix
\section{Appendix}\label{sec:appendix}

\section{Proofs of Lemmas}
\setcounter{theorem}{0}
\begin{lemma}
The log-likelihood function $\mathcal{L}_B$ converges to 
\begin{align*}
\sum_{\textbf{w} \in \mathcal{W}} \sum_{1\leq l \leq L}\!\sum_{| j | \leq \gamma}\Big[ \log p(x_{w_l,w_{l+j}}^{l+j})\! + \! \!\sum_{s=1}^{k}\underset{\substack{u \sim q^-}}{\mathbb{E}}\big[\log p(x_{w_l,u}^{l+j})\big]\Big]
\end{align*}
for large values of $k$.
\end{lemma}
\begin{proof}
Let $q^-(\cdot|w_l)$ be the true conditional distribution of a random walk method for generating context nodes defined over $\mathcal{V}$. Then, it can be written that

\begin{align*}
&\sum_{\textbf{w} \in \mathcal{W}} \sum_{1\leq l \leq L}\sum_{| j | \leq \gamma }\! \log p(x^{l+j}_{w_l,v_{l+j}})\! + \!  \sum_{s=1}^{k}\underset{u \sim q^-}{\mathbb{E}}\big[\log p(x^{l+j}_{w_l,u})\big]\\
\approx& \sum_{\textbf{w} \in \mathcal{W}} \sum_{1\leq l \leq L}\sum_{| j | \leq \gamma } \log p(x^{l+j}_{w_l,v_{l+j}})\! + \! k\frac{1}{k}\sum_{\substack{s=1 \\ u_s \sim q^-}}^{k}\log p(x^{l+j}_{w_l,u_s})\\
=& \sum_{\textbf{w} \in \mathcal{W}} \sum_{1\leq l \leq L}\sum_{ | j | \leq \gamma } \log p(x^{l+j}_{w_l,v_{l+j}}) + \sum_{\substack{s=1 \\ u_s \sim q^-}}^{k}\log p(x^{l+j}_{w_l,u_s})\\
\approx& \sum_{\textbf{w} \in \mathcal{W}} \sum_{1\leq l \leq L}\sum_{  \substack{|  j | \leq \gamma \\ u := v_{l+j}}  } \log p(x^{l+j}_{w_l,u}) + \sum_{\substack{| j | \leq \gamma \\ u :\not= w_{l+j}}}\log p(x^{l+j}_{w_l,u})\\
=& \mathcal{L}_B(\alpha, \beta),
\end{align*}
\noindent where the second line follows from the law of large numbers for the sample size of $k$, and $k$ is set to $|\mathcal{V}|-1$ in the fourth line.
\end{proof}

\begin{lemma}
Let $Z_{w_l,v}$ be independent random variables following Poisson distribution with natural parameter $\eta_{w_l,v}$ defined by $\log(\mathbf{\beta}[w_l]\cdot\mathbf{\alpha}[v])$. Then the objective function of \textsc{EFGE-Bern} model becomes equal to
\begin{align*}
   \sum_{\textbf{w} \in \mathcal{W}} \sum_{1 \leq l \leq L} \Bigg[&\sum_{v \in \mathcal{N}_{\gamma}(w_l)}\log\Big(1 - \exp\big(-\mathbf{\beta}[w_l]^{\top}\!\cdot\!\mathbf{\alpha}[v]\big)\Big) \\
   -& \sum_{v \not\in \mathcal{N}_{\gamma}(w_l)} \mathbf{\beta}[w_l]^{\top}\!\cdot\!\mathbf{\alpha}[v]\Bigg],
\end{align*}
if the model parameter $\pi_{w_l,v}$ is defined by $p(Z_{w_l,v}>0)$.
\end{lemma}
\begin{proof}
Let $y_{w_l,v}$ follow a Bernoulli distribution with parameter $\pi_{w_l,v}$ and it is equal to $1$ if $v \in \mathcal{N}_{\gamma}(w_l)$, and 0 otherwise. Then, the objective function $\mathcal{L}_{B}(\alpha, \beta)$ can be divided into parts as follows: 

\begin{align*}
     \mathcal{L}_{\mathcal{B}} &=\sum_{\textbf{w} \in \mathcal{W}} \sum_{1 \leq l \leq L}\Bigg[\sum_{v \in \mathcal{N}_{\gamma}(w_l)}\!\!\!\log p(y_{w_l,v})\! +\!\!\!\!\!\!\!\! \sum_{v \not\in \mathcal{N}_{\gamma}(w_l)}\!\!\!\!\!\!\log p(y_{w_l,v})\Bigg]\\
     &= \sum_{\textbf{w} \in \mathcal{W}} \sum_{1\leq l\leq L}\Bigg[\sum_{v \in \mathcal{N}_{\gamma}(w_l)}\!\!\!\log\big(1-p(z_{w_l,v}=0)\big)\\ & \quad\quad\quad\quad\quad\quad\quad\quad\quad\quad\quad\quad +\!\!\!\!\!\!\sum_{v \not\in \mathcal{N}_{\gamma}(v_i)}\!\!\log p(z_{w_l,v}=0) \Bigg] \\
     &= \sum_{\textbf{w} \in \mathcal{W}} \sum_{1 \leq l \leq L} \Bigg[\sum_{v \in \mathcal{N}_{\gamma}(w_l)}\!\!\!\log\Big(1 - \exp\big(-\exp(\eta_{w_l,v})\big)\Big) \\&\quad\quad\quad\quad\quad\quad\quad\quad\quad\quad\quad\quad + \sum_{v \not\in \mathcal{N}_{\gamma}(w_l)} \exp(-\eta_{w_l,v})\Bigg]\\
     &= \sum_{\textbf{w} \in \mathcal{W}} \sum_{1 \leq l \leq L} \Bigg[\sum_{v \in \mathcal{N}_{\gamma}(w_l)}\!\!\!\log\Big(1 \!-\! \exp\big(-\mathbf{\beta}[w_l]^{\top}\!\!\!\cdot\!\mathbf{\alpha}[v]\big)\Big) \\& \quad\quad\quad\quad\quad\quad\quad\quad\quad\quad\quad\quad - \sum_{v \not\in \mathcal{N}_{\gamma}(w_l)} \mathbf{\beta}[w_l]^{\top}\!\cdot\!\mathbf{\alpha}[v]\Bigg]
\end{align*}
\end{proof}



\begin{figure*}[h]
  \centering
  \includegraphics[width=.745\linewidth]{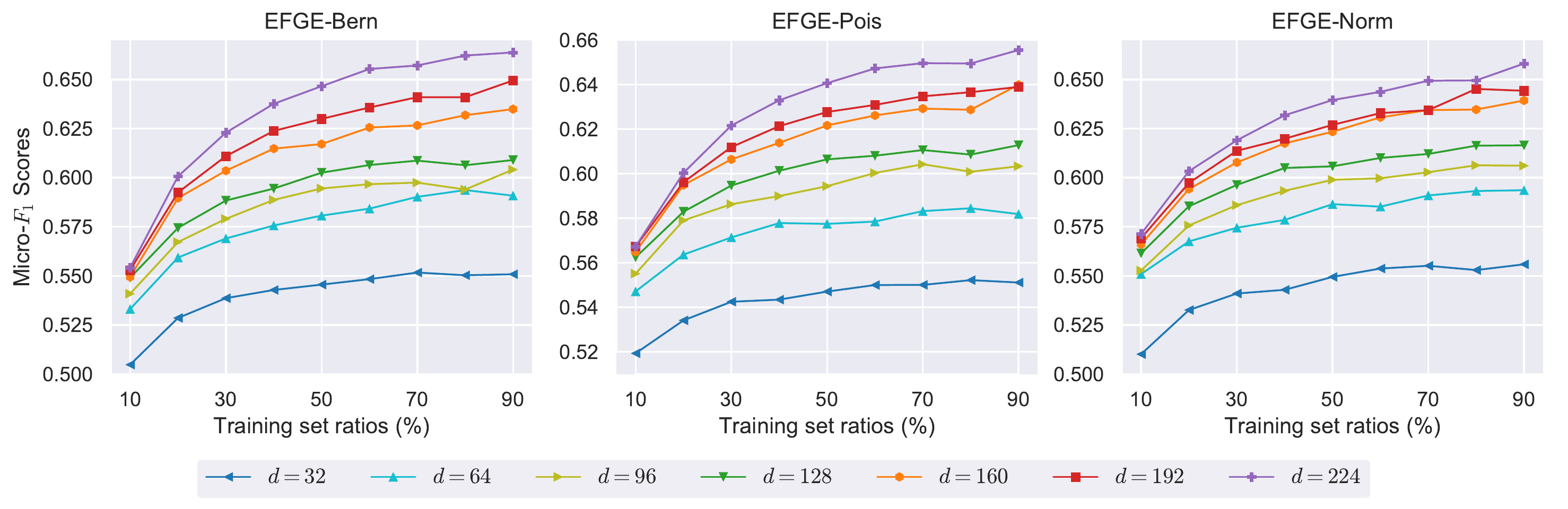}
  \caption{Influence of dimension size over \textsl{CiteSeer} network.}
  \label{fig:dimension_appendix}
\end{figure*}

\begin{figure*}[h]
  \centering
  \includegraphics[width=.745\linewidth]{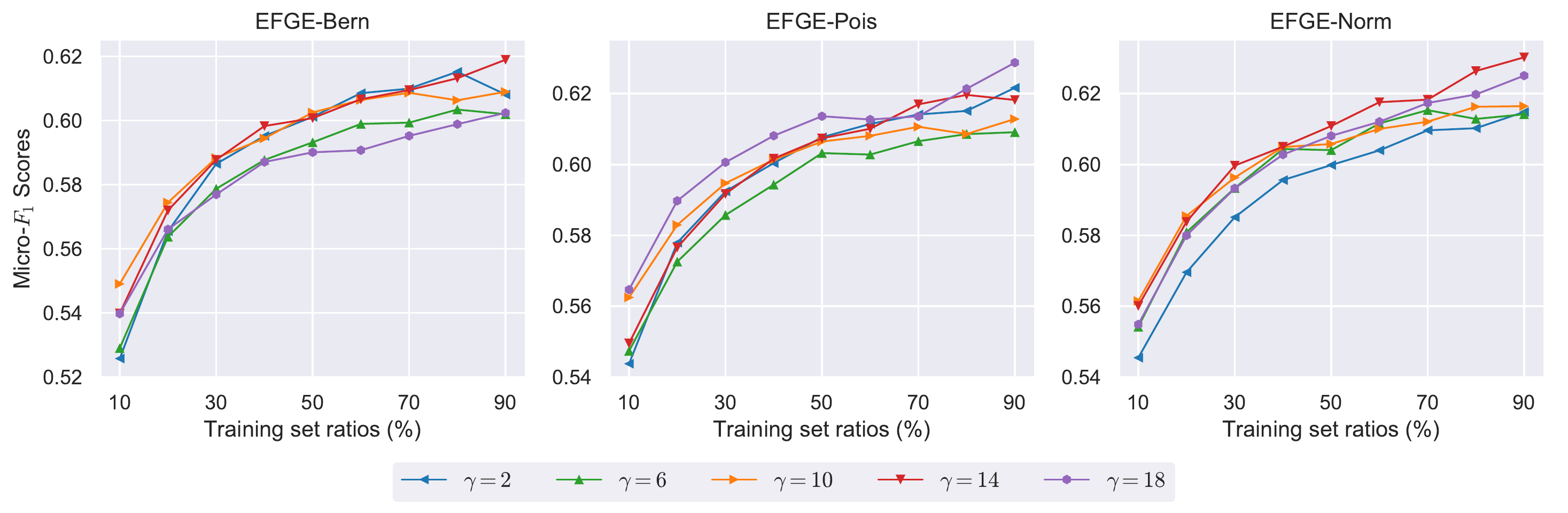}
  \caption{Influence of window size $\gamma$ for the \textsl{CiteSeer} network.}
  \label{fig:window_size_appendix}
\end{figure*}


\section{Dataset Description}
Here we provide a detailed description of the graph datasets\footnote{The datasets can be found at: \url{https://snap.stanford.edu/data} and \url{https://github.com/GTmac/HARP}} used in our study.

\begin{itemize}
	\item \textsl{CiteSeer} 
	is a citation network obtained from the \textit{CiteSeer} library, in which each node corresponds to a paper and the edges indicate  reference relationships among papers. The labels represent the subjects of the paper. 
	
	\item \textsl{Cora} 
	is another citation network constructed from the publications in the machine learning area; the documents are classified into seven categories. 
	
	\item \textsl{DBLP} 
	is a co-authorship graph, where an edge exists between nodes if two authors have co-authored at least one paper. The labels represent the research areas. 
	
	\item \textsl{AstroPh} 
	is another collaboration network built from the papers submitted to the  \textit{ArXiv} repository for the Astro Physics subject area, from January 1993 to April 2003.
	
	\item \textsl{HepTh} 
	network is constructed in a similar way from the papers submitted to \textit{ArXiv} for the \textit{High Energy Physics - Theory} category.
	
	\item \textsl{GrQc} 
	is our last collaboration network which has been constructed from the e-prints submitted to the category of \textit{General Relativity and Quantum Cosmology}.
	
	\item \textsl{Facebook} 
	is a social network extracted from a survey conducted via a \textit{Facebook} application.

\end{itemize}

\begin{figure}[h]
  \centering
  \includegraphics[width=.55\linewidth]{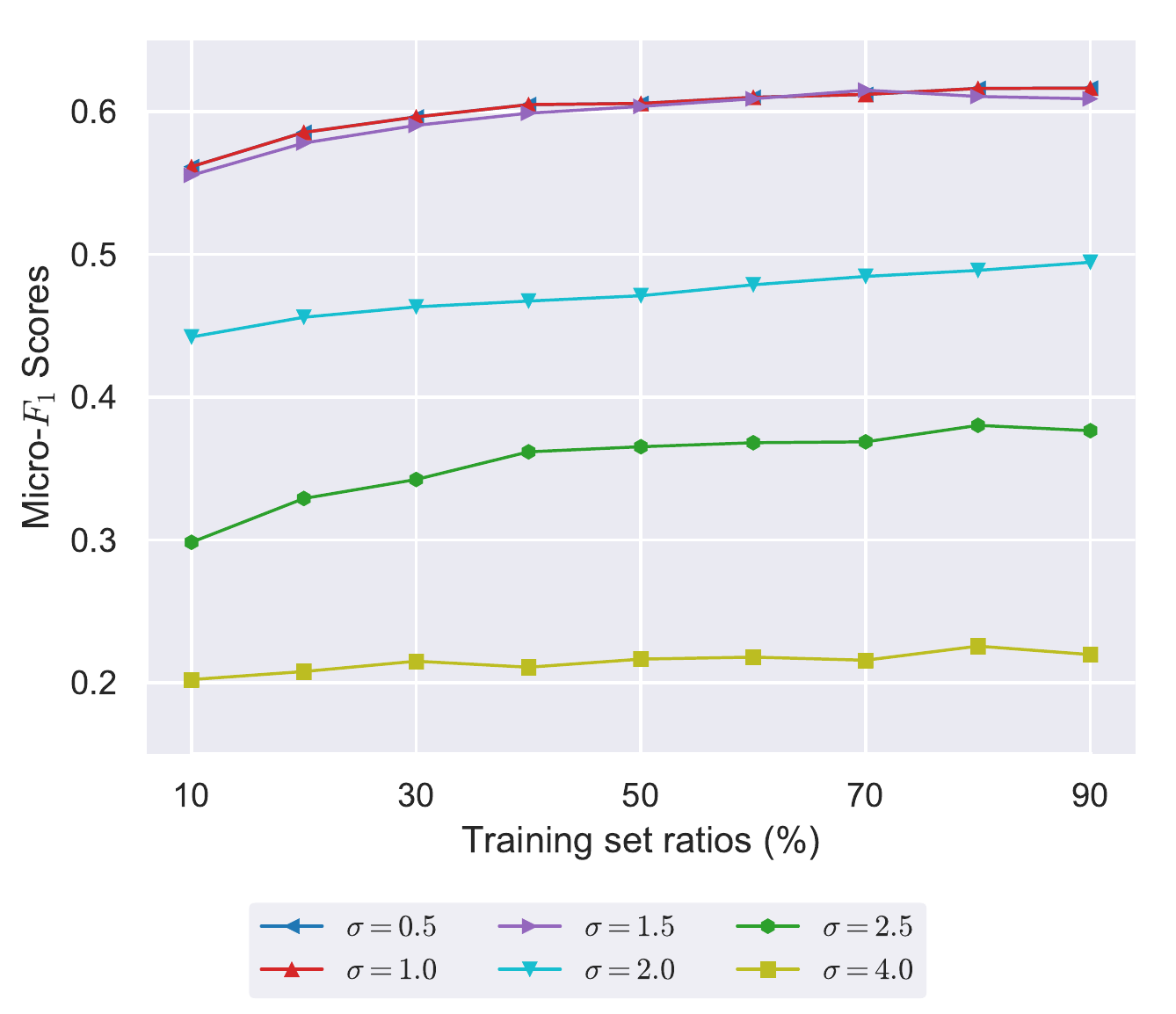}
  \caption{Effect of standard deviation for \textsc{EFGE-Norm}.}
  \label{fig:sigma_effect_appendix}
\end{figure}

\balance

\section{Complementary Experimental Results for Node Classification}

In this section we present complementary experimental results for node classification by reporting the Micro-$F_1$ and Macro-$F_1$ scores for  training size ratios varying from $1\%$ up to $90\%$. The scores are presented in Tables \ref{tab:classification_citeseer_appendix}, \ref{tab:classification_cora_appendix} and \ref{tab:classification_dblp_appendix}.

\section{Parameter Sensitivity}
In this section, we perform some further sensitivity analysis experiments. In particular, we present complementary experiments where we examine the effect of embedding dimension $d$ and the effect of the window size $\gamma$ used to sample context nodes for the different models over different training set ratios. The  results are depicted in Figures \ref{fig:dimension_appendix} and \ref{fig:window_size_appendix}.

\subsubsection*{Effect of standard deviation of \textsc{EFGE-Norm} model}

The \textsc{EFGE-Norm} model has an extra parameter $\sigma$ which can influence the performance of the method. To examine the impact of $\sigma$, we have chosen six different values, performing  experiments over \textsl{CiteSeer} network. Figure \ref{fig:sigma_effect_appendix} depicts how the Micro-$F_1$ scores change for various training set ratios. The results clearly indicate that the model performs well for small values of $\sigma$  --- with the best results obtained for $\sigma=1$. For this reason, we have set this value for all the experiments conducted in the node classification and link prediction tasks.

\begin{table*}[h]
\centering
\resizebox{\textwidth}{!}{%
\begin{tabular}{c|r|cccccccccccccccccc}
 \multicolumn{1}{l}{} &  & \textbf{1\%} & \textbf{2\%} & \textbf{3\%} & \textbf{4\%} & \textbf{5\%} & \textbf{6\%} & \textbf{7\%} & \textbf{8\%} & \textbf{9\%} & \textbf{10\%} & \textbf{20\%} & \textbf{30\%} & \textbf{40\%} & \textbf{50\%} & \textbf{60\%} & \textbf{70\%} & \textbf{80\%} & \textbf{90\%} \\\midrule
\multirow{10}{*}{\rotatebox{90}{\textsc{Baselines}}} &
\multirow{2}{*}{\textsc{DeepWalk}} & 0.373 & 0.416 & 0.443 & 0.460 & 0.471 & 0.489 & 0.496 & 0.505 & 0.506 & 0.517 & 0.548 & 0.566 & 0.576 & 0.584 & 0.590 & 0.595 & 0.591 & 0.592 \\
 &  & \multicolumn{1}{c}{0.318} & \multicolumn{1}{c}{0.367} & \multicolumn{1}{c}{0.403} & \multicolumn{1}{c}{0.418} & \multicolumn{1}{c}{0.429} & \multicolumn{1}{c}{0.446} & \multicolumn{1}{c}{0.454} & \multicolumn{1}{c}{0.463} & \multicolumn{1}{c}{0.464} & \multicolumn{1}{c}{0.474} & \multicolumn{1}{c}{0.505} & \multicolumn{1}{c}{0.521} & \multicolumn{1}{c}{0.529} & \multicolumn{1}{c}{0.537} & \multicolumn{1}{c}{0.542} & \multicolumn{1}{c}{0.547} & \multicolumn{1}{c}{0.541} & \multicolumn{1}{c}{0.543} \\
 & \multirow{2}{*}{\textsc{Node2Vec}} & \cellcolor[HTML]{F2F3F4}0.405 & \cellcolor[HTML]{F2F3F4}\cellcolor[HTML]{F2F3F4}0.450 & \cellcolor[HTML]{F2F3F4}0.475 & \cellcolor[HTML]{F2F3F4}0.491 & \cellcolor[HTML]{F2F3F4}0.500 & \cellcolor[HTML]{F2F3F4}0.517 & \cellcolor[HTML]{F2F3F4}0.524 & \cellcolor[HTML]{F2F3F4}0.530 & \cellcolor[HTML]{F2F3F4}0.537 & \cellcolor[HTML]{F2F3F4}0.541 & \cellcolor[HTML]{F2F3F4}0.570 & \cellcolor[HTML]{F2F3F4}0.585 & \cellcolor[HTML]{F2F3F4}0.590 & \cellcolor[HTML]{F2F3F4}0.597 & \cellcolor[HTML]{F2F3F4}0.598 & \cellcolor[HTML]{F2F3F4}0.601 & \cellcolor[HTML]{F2F3F4}0.596 & \cellcolor[HTML]{F2F3F4}0.599 \\
 &  & \multicolumn{1}{c}{\cellcolor[HTML]{F2F3F4}0.339} & \multicolumn{1}{c}{\cellcolor[HTML]{F2F3F4}0.396} & \multicolumn{1}{c}{\cellcolor[HTML]{F2F3F4}0.424} & \multicolumn{1}{c}{\cellcolor[HTML]{F2F3F4}0.441} & \multicolumn{1}{c}{\cellcolor[HTML]{F2F3F4}0.452} & \multicolumn{1}{c}{\cellcolor[HTML]{F2F3F4}0.470} & \multicolumn{1}{c}{\cellcolor[HTML]{F2F3F4}0.478} & \multicolumn{1}{c}{\cellcolor[HTML]{F2F3F4}0.483} & \multicolumn{1}{c}{\cellcolor[HTML]{F2F3F4}0.491} & \multicolumn{1}{c}{\cellcolor[HTML]{F2F3F4}0.494} & \multicolumn{1}{c}{\cellcolor[HTML]{F2F3F4}0.525} & \multicolumn{1}{c}{\cellcolor[HTML]{F2F3F4}0.537} & \multicolumn{1}{c}{\cellcolor[HTML]{F2F3F4}0.542} & \multicolumn{1}{c}{\cellcolor[HTML]{F2F3F4}0.549} & \multicolumn{1}{c}{\cellcolor[HTML]{F2F3F4}0.550} & \multicolumn{1}{c}{\cellcolor[HTML]{F2F3F4}0.553} & \multicolumn{1}{c}{\cellcolor[HTML]{F2F3F4}0.548} & \multicolumn{1}{c}{\cellcolor[HTML]{F2F3F4}0.551} \\
 & \multirow{2}{*}{\textsc{LINE}} & 0.273 & 0.323 & 0.362 & 0.387 & 0.406 & 0.423 & 0.440 & 0.451 & 0.456 & 0.466 & 0.513 & 0.532 & 0.543 & 0.551 & 0.556 & 0.560 & 0.568 & 0.564 \\
 &  & 0.204 & 0.268 & 0.311 & 0.338 & 0.360 & 0.372 & 0.390 & 0.399 & 0.407 & 0.414 & 0.459 & 0.480 & 0.492 & 0.498 & 0.505 & 0.505 & 0.514 & 0.513 \\
 & \multirow{2}{*}{\textsc{HOPE}} &\cellcolor[HTML]{F2F3F4} 0.194 &\cellcolor[HTML]{F2F3F4} 0.196 & \cellcolor[HTML]{F2F3F4}0.202 & \cellcolor[HTML]{F2F3F4}0.205 & \cellcolor[HTML]{F2F3F4}0.208 & \cellcolor[HTML]{F2F3F4}0.210 & \cellcolor[HTML]{F2F3F4}0.216 & \cellcolor[HTML]{F2F3F4}0.204 & \cellcolor[HTML]{F2F3F4}0.216 & \cellcolor[HTML]{F2F3F4}0.219 & \cellcolor[HTML]{F2F3F4}0.228 & \cellcolor[HTML]{F2F3F4}0.256 & \cellcolor[HTML]{F2F3F4}0.267 &\cellcolor[HTML]{F2F3F4} 0.277 & \cellcolor[HTML]{F2F3F4}0.293 & \cellcolor[HTML]{F2F3F4}0.299 &\cellcolor[HTML]{F2F3F4} 0.300 & \cellcolor[HTML]{F2F3F4}0.320 \\
 &  & \multicolumn{1}{c}{\cellcolor[HTML]{F2F3F4}0.060} & \multicolumn{1}{c}{\cellcolor[HTML]{F2F3F4}0.060} & \multicolumn{1}{c}{\cellcolor[HTML]{F2F3F4}0.063} & \multicolumn{1}{c}{\cellcolor[HTML]{F2F3F4}0.062} & \multicolumn{1}{c}{\cellcolor[HTML]{F2F3F4}0.066} & \multicolumn{1}{c}{\cellcolor[HTML]{F2F3F4}0.068} & \multicolumn{1}{c}{\cellcolor[HTML]{F2F3F4}0.079} & \multicolumn{1}{c}{\cellcolor[HTML]{F2F3F4}0.064} & \multicolumn{1}{c}{\cellcolor[HTML]{F2F3F4}0.075} & \multicolumn{1}{c}{\cellcolor[HTML]{F2F3F4}0.078} & \multicolumn{1}{c}{\cellcolor[HTML]{F2F3F4}0.094} & \multicolumn{1}{c}{\cellcolor[HTML]{F2F3F4}0.127} & \multicolumn{1}{c}{\cellcolor[HTML]{F2F3F4}0.136} & \multicolumn{1}{c}{\cellcolor[HTML]{F2F3F4}0.150} & \multicolumn{1}{c}{\cellcolor[HTML]{F2F3F4}0.168} & \multicolumn{1}{c}{\cellcolor[HTML]{F2F3F4}0.178} & \multicolumn{1}{c}{\cellcolor[HTML]{F2F3F4}0.183} & \multicolumn{1}{c}{\cellcolor[HTML]{F2F3F4}0.205}\\
 & \multirow{2}{*}{\textsc{NetMF}} & 0.379 & 0.451 & 0.472 & 0.496 & 0.515 & 0.526 & 0.533 & 0.540 & 0.544 & 0.552 & 0.578 & 0.590 & 0.596 & 0.603 & 0.605 & 0.604 & 0.611 & 0.608 \\
 & & 0.315 & 0.400 & 0.423 & 0.445 & 0.464 & 0.477 & 0.486 & 0.490 & 0.497 & 0.503 & 0.529 & 0.542 & 0.546 & 0.553 & 0.554 & 0.552 & 0.560 & 0.554
 \\\midrule
\multirow{5}{*}{\rotatebox{90}{\textsc{EFGE}}} & \multirow{2}{*}{\textsc{EFGE-Bern}} & 0.411 & 0.461 & 0.487 & 0.493 & 0.513 & 0.517 & 0.528 & 0.536 & 0.543 & 0.549 & 0.574 & 0.588 & 0.594 & 0.603 & 0.606 & 0.609 & 0.606 & 0.609 \\
 &  & \multicolumn{1}{c}{0.345} & \multicolumn{1}{c}{0.410} & \multicolumn{1}{c}{0.436} & \multicolumn{1}{c}{0.446} & \multicolumn{1}{c}{0.462} & \multicolumn{1}{c}{0.468} & \multicolumn{1}{c}{0.481} & \multicolumn{1}{c}{0.489} & \multicolumn{1}{c}{0.496} & \multicolumn{1}{c}{0.502} & \multicolumn{1}{c}{0.530} & \multicolumn{1}{c}{0.545} & \multicolumn{1}{c}{0.549} & \multicolumn{1}{c}{0.558} & \multicolumn{1}{c}{0.563} & \multicolumn{1}{c}{0.564} & \multicolumn{1}{c}{0.561} & \multicolumn{1}{c}{0.563} \\
 & \multirow{2}{*}{\textsc{EFGE-Pois}} & \cellcolor[HTML]{F2F3F4}\textbf{0.449} & \cellcolor[HTML]{F2F3F4}0.484 &\cellcolor[HTML]{F2F3F4} 0.500 &\cellcolor[HTML]{F2F3F4} 0.514 & \cellcolor[HTML]{F2F3F4}0.530 &\cellcolor[HTML]{F2F3F4} 0.537 & \cellcolor[HTML]{F2F3F4}0.544 & \cellcolor[HTML]{F2F3F4}0.551 &\cellcolor[HTML]{F2F3F4} 0.555 & \cellcolor[HTML]{F2F3F4}\textbf{0.562} &\cellcolor[HTML]{F2F3F4} 0.583 & \cellcolor[HTML]{F2F3F4}0.595 & \cellcolor[HTML]{F2F3F4}0.601 & \cellcolor[HTML]{F2F3F4}\textbf{0.606} & \cellcolor[HTML]{F2F3F4}0.608 & \cellcolor[HTML]{F2F3F4}0.611 & \cellcolor[HTML]{F2F3F4}0.609 & \cellcolor[HTML]{F2F3F4}0.613 \\
 &  & \multicolumn{1}{c}{\cellcolor[HTML]{F2F3F4}\textbf{0.384}} & \multicolumn{1}{c}{\cellcolor[HTML]{F2F3F4}0.425} & \multicolumn{1}{c}{\cellcolor[HTML]{F2F3F4}0.445} & \multicolumn{1}{c}{\cellcolor[HTML]{F2F3F4}0.461} & \multicolumn{1}{c}{\cellcolor[HTML]{F2F3F4}\textbf{0.477}} & \multicolumn{1}{c}{\cellcolor[HTML]{F2F3F4}0.482} & \multicolumn{1}{c}{\cellcolor[HTML]{F2F3F4}0.491} & \multicolumn{1}{c}{\cellcolor[HTML]{F2F3F4}\textbf{0.500}} & \multicolumn{1}{c}{\cellcolor[HTML]{F2F3F4}\textbf{0.504}} & \multicolumn{1}{c}{\cellcolor[HTML]{F2F3F4}\textbf{0.512}} & \multicolumn{1}{c}{\cellcolor[HTML]{F2F3F4}0.533} & \multicolumn{1}{c}{\cellcolor[HTML]{F2F3F4}0.547} & \multicolumn{1}{c}{\cellcolor[HTML]{F2F3F4}0.552} & \multicolumn{1}{c}{\cellcolor[HTML]{F2F3F4}0.558} & \multicolumn{1}{c}{\cellcolor[HTML]{F2F3F4}0.559} & \multicolumn{1}{c}{\cellcolor[HTML]{F2F3F4}0.563} & \multicolumn{1}{c}{\cellcolor[HTML]{F2F3F4}0.560} & \multicolumn{1}{c}{\cellcolor[HTML]{F2F3F4}0.564} \\
 & \multirow{2}{*}{\textsc{EFGE-Norm}} & \multicolumn{1}{c}{0.434} & \multicolumn{1}{c}{\textbf{0.493}} & \multicolumn{1}{c}{\textbf{0.510}} & \multicolumn{1}{c}{\textbf{0.525}} & \multicolumn{1}{c}{\textbf{0.531}} & \multicolumn{1}{c}{\textbf{0.542}} & \multicolumn{1}{c}{\textbf{0.546}} & \multicolumn{1}{c}{\textbf{0.553}} & \multicolumn{1}{c}{\textbf{0.557}} & \multicolumn{1}{c}{0.561} & \multicolumn{1}{c}{\textbf{0.585}} & \multicolumn{1}{c}{\textbf{0.596}} & \multicolumn{1}{c}{\textbf{0.605}} & \multicolumn{1}{c}{\textbf{0.606}} & \multicolumn{1}{c}{\textbf{0.610}} & \multicolumn{1}{c}{\textbf{0.612}} & \multicolumn{1}{c}{\textbf{0.616}} & \multicolumn{1}{c}{\textbf{0.616}} \\
 & & 0.361 & \textbf{0.431} & \textbf{0.450} & \textbf{0.467} & 0.474 & \textbf{0.485} & \textbf{0.492} & \textbf{0.500} & 0.501 & 0.509 & \textbf{0.534} & \textbf{0.547} & \textbf{0.558} & \textbf{0.560} & \textbf{0.566} & \textbf{0.567} & \textbf{0.572} & \textbf{0.575}\\\bottomrule
\end{tabular}%
}
\caption{Node classification for varying training sizes for the \textsl{CiteSeer} network. For each method, the first row indicates the Micro-$F_1$ scores and the second one shows the Macro-$F_1$ scores.}
\label{tab:classification_citeseer_appendix}
\end{table*}

\begin{table*}[h]
\centering
\resizebox{\textwidth}{!}{%
\begin{tabular}{c|r|cccccccccccccccccc}
\multicolumn{1}{c}{} &  & \textbf{1\%} & \textbf{2\%} & \textbf{3\%} & \textbf{4\%} & \textbf{5\%} & \textbf{6\%} & \textbf{7\%} & \textbf{8\%} & \textbf{9\%} & \textbf{10\%} & \textbf{20\%} & \textbf{30\%} & \textbf{40\%} & \textbf{50\%} & \textbf{60\%} & \textbf{70\%} & \textbf{80\%} & \textbf{90\%} \\\midrule
\multirow{10}{*}{\rotatebox{90}{\textsc{Baselines}}} & \multirow{2}{*}{\textsc{Deepwalk}} & 0.531 & 0.621 & 0.667 & 0.689 & 0.703 & 0.715 & 0.727 & 0.732 & 0.744 & 0.747 & 0.784 & 0.802 & 0.810 & 0.819 & 0.822 & 0.826 & 0.826 & 0.833 \\
 &  & \multicolumn{1}{c}{0.455} & \multicolumn{1}{c}{0.571} & \multicolumn{1}{c}{0.638} & \multicolumn{1}{c}{0.666} & \multicolumn{1}{c}{0.681} & \multicolumn{1}{c}{0.698} & \multicolumn{1}{c}{0.711} & \multicolumn{1}{c}{0.717} & \multicolumn{1}{c}{0.730} & \multicolumn{1}{c}{0.734} & \multicolumn{1}{c}{0.774} & \multicolumn{1}{c}{0.792} & \multicolumn{1}{c}{0.800} & \multicolumn{1}{c}{0.809} & \multicolumn{1}{c}{0.813} & \multicolumn{1}{c}{0.816} & \multicolumn{1}{c}{0.815} & \multicolumn{1}{c}{0.825} \\
 & \multirow{2}{*}{\textsc{Node2Vec}} & \cellcolor[HTML]{F2F3F4}0.575 &\cellcolor[HTML]{F2F3F4}0.656 & \cellcolor[HTML]{F2F3F4}0.696 &\cellcolor[HTML]{F2F3F4}0.714 & \cellcolor[HTML]{F2F3F4}0.731 &\cellcolor[HTML]{F2F3F4}0.743 & \cellcolor[HTML]{F2F3F4}0.746 & \cellcolor[HTML]{F2F3F4}0.757 &\cellcolor[HTML]{F2F3F4}0.764 &\cellcolor[HTML]{F2F3F4}0.769 &\cellcolor[HTML]{F2F3F4}0.799 & \cellcolor[HTML]{F2F3F4}0.815 & \cellcolor[HTML]{F2F3F4}0.824 & \cellcolor[HTML]{F2F3F4}0.831 & \cellcolor[HTML]{F2F3F4}0.835 & \cellcolor[HTML]{F2F3F4}0.839 & \cellcolor[HTML]{F2F3F4}\textbf{0.842} & \cellcolor[HTML]{F2F3F4}0.841 \\
 &  & \multicolumn{1}{c}{\cellcolor[HTML]{F2F3F4}0.501} & \multicolumn{1}{c}{\cellcolor[HTML]{F2F3F4}0.605} & \multicolumn{1}{c}{\cellcolor[HTML]{F2F3F4}0.665} & \multicolumn{1}{c}{\cellcolor[HTML]{F2F3F4}0.687} & \multicolumn{1}{c}{\cellcolor[HTML]{F2F3F4}0.713} & \multicolumn{1}{c}{\cellcolor[HTML]{F2F3F4}0.723} & \multicolumn{1}{c}{\cellcolor[HTML]{F2F3F4}0.730} & \multicolumn{1}{c}{\cellcolor[HTML]{F2F3F4}0.741} & \multicolumn{1}{c}{\cellcolor[HTML]{F2F3F4}0.750} & \multicolumn{1}{c}{\cellcolor[HTML]{F2F3F4}0.755} & \multicolumn{1}{c}{\cellcolor[HTML]{F2F3F4}0.786} & \multicolumn{1}{c}{\cellcolor[HTML]{F2F3F4}0.803} & \multicolumn{1}{c}{\cellcolor[HTML]{F2F3F4}0.812} & \multicolumn{1}{c}{\cellcolor[HTML]{F2F3F4}0.819} & \multicolumn{1}{c}{\cellcolor[HTML]{F2F3F4}0.823} & \multicolumn{1}{c}{\cellcolor[HTML]{F2F3F4}0.826} & \multicolumn{1}{c}{\cellcolor[HTML]{F2F3F4}0.830} & \multicolumn{1}{c}{\cellcolor[HTML]{F2F3F4}0.825} \\
 & \multirow{2}{*}{\textsc{LINE}} & 0.384 & 0.450 & 0.506 & 0.544 & 0.568 & 0.590 & 0.618 & 0.633 & 0.648 & 0.661 & 0.723 & 0.746 & 0.758 & 0.765 & 0.770 & 0.774 & 0.775 & 0.775 \\
 &  & 0.272 & 0.364 & 0.435 & 0.491 & 0.523 & 0.555 & 0.588 & 0.607 & 0.626 & 0.642 & 0.713 & 0.736 & 0.747 & 0.755 & 0.759 & 0.762 & 0.766 & 0.764 \\
 & \multirow{2}{*}{\textsc{HOPE}} & \cellcolor[HTML]{F2F3F4}0.260 & \cellcolor[HTML]{F2F3F4}0.277 &\cellcolor[HTML]{F2F3F4} 0.297 & \cellcolor[HTML]{F2F3F4}0.302 & \cellcolor[HTML]{F2F3F4}0.304 & \cellcolor[HTML]{F2F3F4}0.299 & \cellcolor[HTML]{F2F3F4}0.302 & \cellcolor[HTML]{F2F3F4}0.302 & \cellcolor[HTML]{F2F3F4}\cellcolor[HTML]{F2F3F4}0.302 & \cellcolor[HTML]{F2F3F4}0.302 & \cellcolor[HTML]{F2F3F4}0.303 & \cellcolor[HTML]{F2F3F4}0.301 &\cellcolor[HTML]{F2F3F4}0.303 & \cellcolor[HTML]{F2F3F4}0.302 & \cellcolor[HTML]{F2F3F4}0.303 & \cellcolor[HTML]{F2F3F4}0.303 & \cellcolor[HTML]{F2F3F4}0.303 & \cellcolor[HTML]{F2F3F4}0.302 \\
 &  & \multicolumn{1}{c}{\cellcolor[HTML]{F2F3F4}0.065} & \multicolumn{1}{c}{\cellcolor[HTML]{F2F3F4}0.068} & \multicolumn{1}{c}{\cellcolor[HTML]{F2F3F4}0.067} & \multicolumn{1}{c}{\cellcolor[HTML]{F2F3F4}0.066} & \multicolumn{1}{c}{\cellcolor[HTML]{F2F3F4}0.068} & \multicolumn{1}{c}{\cellcolor[HTML]{F2F3F4}0.066} & \multicolumn{1}{c}{\cellcolor[HTML]{F2F3F4}0.066} & \multicolumn{1}{c}{\cellcolor[HTML]{F2F3F4}0.066} & \multicolumn{1}{c}{\cellcolor[HTML]{F2F3F4}0.066} & \multicolumn{1}{c}{\cellcolor[HTML]{F2F3F4}0.066} & \multicolumn{1}{c}{\cellcolor[HTML]{F2F3F4}0.067} & \multicolumn{1}{c}{\cellcolor[HTML]{F2F3F4}0.067} & \multicolumn{1}{c}{\cellcolor[HTML]{F2F3F4}0.067} & \multicolumn{1}{c}{\cellcolor[HTML]{F2F3F4}0.067} & \multicolumn{1}{c}{\cellcolor[HTML]{F2F3F4}0.067} & \multicolumn{1}{c}{\cellcolor[HTML]{F2F3F4}0.068} & \multicolumn{1}{c}{\cellcolor[HTML]{F2F3F4}0.070} & \multicolumn{1}{c}{\cellcolor[HTML]{F2F3F4}0.072}\\
 & \multirow{2}{*}{\textsc{NetMF}} & 0.534 & 0.636 & 0.693 & 0.716 & 0.735 & 0.748 & 0.757 & 0.767 & 0.770 & 0.773 & \textbf{0.807} & \textbf{0.821} & \textbf{0.828} & \textbf{0.834} & \textbf{0.839} & \textbf{0.841} & 0.839 & \textbf{0.844} \\
 & & 0.461 & 0.591 & 0.667 & 0.694 & 0.717 & 0.731 & 0.741 & 0.751 & 0.757 & 0.760 & \textbf{0.797} & \textbf{0.811} & \textbf{0.819} & \textbf{0.824} & \textbf{0.830} & \textbf{0.832} & \textbf{0.831} & \textbf{0.835}
 \\\midrule
\multirow{6}{*}{\rotatebox{90}{\textsc{EFGE}}} & \multirow{2}{*}{\textsc{EFGE-Bern}} & 0.593 & 0.668 & 0.703 & 0.720 & 0.738 & 0.743 & 0.751 & 0.759 & 0.760 & 0.767 & 0.792 & 0.808 & 0.815 & 0.823 & 0.828 & 0.834 & 0.837 & 0.838 \\
 &  & \multicolumn{1}{c}{0.507} & \multicolumn{1}{c}{0.622} & \multicolumn{1}{c}{0.672} & \multicolumn{1}{c}{0.695} & \multicolumn{1}{c}{0.718} & \multicolumn{1}{c}{0.723} & \multicolumn{1}{c}{0.735} & \multicolumn{1}{c}{0.744} & \multicolumn{1}{c}{0.743} & \multicolumn{1}{c}{0.754} & \multicolumn{1}{c}{0.780} & \multicolumn{1}{c}{0.798} & \multicolumn{1}{c}{0.806} & \multicolumn{1}{c}{0.814} & \multicolumn{1}{c}{0.819} & \multicolumn{1}{c}{0.823} & \multicolumn{1}{c}{0.825} & \multicolumn{1}{c}{0.824} \\
 & \multirow{2}{*}{\textsc{EFGE-Pois}} & \cellcolor[HTML]{F2F3F4}\textbf{0.605} & \cellcolor[HTML]{F2F3F4}0.680 & \cellcolor[HTML]{F2F3F4}0.714 & \cellcolor[HTML]{F2F3F4} 0.733 & \cellcolor[HTML]{F2F3F4}0.739 & \cellcolor[HTML]{F2F3F4}0.746 & \cellcolor[HTML]{F2F3F4}0.752 & \cellcolor[HTML]{F2F3F4}0.759 & \cellcolor[HTML]{F2F3F4}0.761 & \cellcolor[HTML]{F2F3F4}0.765 & \cellcolor[HTML]{F2F3F4}0.791 & \cellcolor[HTML]{F2F3F4}0.802 & \cellcolor[HTML]{F2F3F4}0.809 & \cellcolor[HTML]{F2F3F4}0.814 & \cellcolor[HTML]{F2F3F4} 0.817 & \cellcolor[HTML]{F2F3F4}0.820 & \cellcolor[HTML]{F2F3F4}0.824 & \cellcolor[HTML]{F2F3F4}0.825 \\
 &  & \multicolumn{1}{c}{\cellcolor[HTML]{F2F3F4}\textbf{0.512}} & \multicolumn{1}{c}{\cellcolor[HTML]{F2F3F4}\textbf{0.630}} & \multicolumn{1}{c}{\cellcolor[HTML]{F2F3F4}\textbf{0.685}} & \multicolumn{1}{c}{\cellcolor[HTML]{F2F3F4}0.709} & \multicolumn{1}{c}{\cellcolor[HTML]{F2F3F4}0.715} & \multicolumn{1}{c}{\cellcolor[HTML]{F2F3F4}0.731} & \multicolumn{1}{c}{\cellcolor[HTML]{F2F3F4}0.737} & \multicolumn{1}{c}{\cellcolor[HTML]{F2F3F4}0.744} & \multicolumn{1}{c}{\cellcolor[HTML]{F2F3F4}0.748} & \multicolumn{1}{c}{\cellcolor[HTML]{F2F3F4}0.752} & \multicolumn{1}{c}{\cellcolor[HTML]{F2F3F4}0.780} & \multicolumn{1}{c}{\cellcolor[HTML]{F2F3F4}0.792} & \multicolumn{1}{c}{\cellcolor[HTML]{F2F3F4}0.798} & \multicolumn{1}{c}{\cellcolor[HTML]{F2F3F4}0.803} & \multicolumn{1}{c}{\cellcolor[HTML]{F2F3F4}0.807} & \multicolumn{1}{c}{\cellcolor[HTML]{F2F3F4}0.810} & \multicolumn{1}{c}{\cellcolor[HTML]{F2F3F4}0.815} & \multicolumn{1}{c}{\cellcolor[HTML]{F2F3F4}0.813} \\
 & \multirow{2}{*}{\textsc{EFGE-Norm}} & 0.601 & \multicolumn{1}{c}{\textbf{0.682}} & \multicolumn{1}{c}{\textbf{0.720}} & \multicolumn{1}{c}{\textbf{0.743}} & \multicolumn{1}{c}{\textbf{0.754}} & \multicolumn{1}{c}{\textbf{0.760}} & \multicolumn{1}{c}{\textbf{0.765}} & \multicolumn{1}{c}{\textbf{0.770}} & \multicolumn{1}{c}{\textbf{0.776}} & \multicolumn{1}{c}{\textbf{0.780}} & \multicolumn{1}{c}{0.799} & \multicolumn{1}{c}{0.810} & \multicolumn{1}{c}{0.814} & \multicolumn{1}{c}{0.824} & \multicolumn{1}{c}{0.824} & \multicolumn{1}{c}{0.827} & \multicolumn{1}{c}{0.832} & \multicolumn{1}{c}{0.839} \\
 & & \textbf{0.512} & 0.626 & \textbf{0.685} & \textbf{0.711} & \textbf{0.730} & \textbf{0.741} & \textbf{0.746} & \textbf{0.754} & \textbf{0.760} & \textbf{0.766} & 0.785 & 0.796 & 0.800 & 0.810 & 0.810 & 0.812 & 0.818 & 0.824\\\bottomrule
\end{tabular}%
}
\caption{Node classification for varying training sizes for the \textsl{Cora} network. For each method, the first row indicates the Micro-$F_1$ scores and the second one shows the Macro-$F_1$ scores.}
\label{tab:classification_cora_appendix}
\end{table*}

\begin{table*}[h]
\centering
\resizebox{\textwidth}{!}{%
\begin{tabular}{c|r|cccccccccccccccccc}
 &  & \textbf{1\%} & \textbf{2\%} & \textbf{3\%} & \textbf{4\%} & \textbf{5\%} & \textbf{6\%} & \textbf{7\%} & \textbf{8\%} & \textbf{9\%} & \textbf{10\%} & \textbf{20\%} & \textbf{30\%} & \textbf{40\%} & \textbf{50\%} & \textbf{60\%} & \textbf{70\%} & \textbf{80\%} & \textbf{90\%} \\\midrule
\multirow{10}{*}{\rotatebox{90}{\textsc{Baselines}}} & \multirow{2}{*}{\textsc{Deepwalk}} & 0.517 & 0.545 & 0.569 & 0.585 & 0.593 & 0.600 & 0.604 & 0.608 & 0.610 & 0.613 & 0.622 & 0.626 & 0.627 & 0.628 & 0.627 & 0.628 & 0.629 & 0.633 \\
 &  & 0.462 & 0.492 & 0.510 & 0.522 & 0.529 & 0.535 & 0.537 & 0.542 & 0.543 & 0.545 & 0.552 & 0.555 & 0.556 & 0.556 & 0.555 & 0.556 & 0.557 & 0.559 \\
 & \multirow{2}{*}{\textsc{Node2Vec}} &\cellcolor[HTML]{F2F3F4}0.557 &\cellcolor[HTML]{F2F3F4}0.575 & \cellcolor[HTML]{F2F3F4}0.590 & \cellcolor[HTML]{F2F3F4}0.600 & \cellcolor[HTML]{F2F3F4}0.605 & \cellcolor[HTML]{F2F3F4}0.611 & \cellcolor[HTML]{F2F3F4}0.615 & \cellcolor[HTML]{F2F3F4}0.619 & \cellcolor[HTML]{F2F3F4}0.621 & \cellcolor[HTML]{F2F3F4}0.622 & \cellcolor[HTML]{F2F3F4}0.632 & \cellcolor[HTML]{F2F3F4}0.636 & \cellcolor[HTML]{F2F3F4}0.638 &
 \cellcolor[HTML]{F2F3F4}0.638 &
 \cellcolor[HTML]{F2F3F4}0.640 &
 \cellcolor[HTML]{F2F3F4}0.639 & \cellcolor[HTML]{F2F3F4}0.640 & \cellcolor[HTML]{F2F3F4}0.639 \\
 &  & \cellcolor[HTML]{F2F3F4}0.497 & \cellcolor[HTML]{F2F3F4}0.517 &\cellcolor[HTML]{F2F3F4} 0.532 & \cellcolor[HTML]{F2F3F4}0.541 & \cellcolor[HTML]{F2F3F4}0.545 & \cellcolor[HTML]{F2F3F4}0.552 & \cellcolor[HTML]{F2F3F4}0.554 & \cellcolor[HTML]{F2F3F4}\textbf{0.558} & \cellcolor[HTML]{F2F3F4}0.558 & \cellcolor[HTML]{F2F3F4}0.560 & \cellcolor[HTML]{F2F3F4}\textbf{0.569} & \cellcolor[HTML]{F2F3F4}\textbf{0.571} & \cellcolor[HTML]{F2F3F4}\textbf{0.572} & \cellcolor[HTML]{F2F3F4}0.572 & \cellcolor[HTML]{F2F3F4}\textbf{0.574} &
 \cellcolor[HTML]{F2F3F4}0.573 & \cellcolor[HTML]{F2F3F4}\textbf{0.574} &\cellcolor[HTML]{F2F3F4}\textbf{0.573} \\
 & \multirow{2}{*}{\textsc{LINE}} & 0.525 & 0.554 & 0.570 & 0.580 & 0.587 & 0.590 & 0.594 & 0.597 & 0.600 & 0.603 & 0.613 & 0.618 & 0.619 & 0.621 & 0.621 & 0.623 & 0.623 & 0.623 \\
 &  & 0.434 & 0.475 & 0.499 & 0.510 & 0.519 & 0.521 & 0.526 & 0.530 & 0.533 & 0.535 & 0.548 & 0.552 & 0.554 & 0.556 & 0.555 & 0.557 & 0.558 & 0.559 \\
 & \multirow{2}{*}{\textsc{HOPE}} & \cellcolor[HTML]{F2F3F4}0.376 & \cellcolor[HTML]{F2F3F4}0.379 & \cellcolor[HTML]{F2F3F4}0.379 & \cellcolor[HTML]{F2F3F4}0.378 & \cellcolor[HTML]{F2F3F4}0.378 & \cellcolor[HTML]{F2F3F4}0.379 & \cellcolor[HTML]{F2F3F4}\cellcolor[HTML]{F2F3F4}0.379 & \cellcolor[HTML]{F2F3F4}0.379 & \cellcolor[HTML]{F2F3F4}0.378 & \cellcolor[HTML]{F2F3F4}0.379 & \cellcolor[HTML]{F2F3F4}0.378 & \cellcolor[HTML]{F2F3F4}0.379 & \cellcolor[HTML]{F2F3F4}0.379 & \cellcolor[HTML]{F2F3F4}0.379 & \cellcolor[HTML]{F2F3F4}0.379 & \cellcolor[HTML]{F2F3F4}0.378 & \cellcolor[HTML]{F2F3F4}0.379 & \cellcolor[HTML]{F2F3F4}0.380 \\
 &  & \cellcolor[HTML]{F2F3F4}0.137 & \cellcolor[HTML]{F2F3F4}0.137 & \cellcolor[HTML]{F2F3F4}0.137 & \cellcolor[HTML]{F2F3F4}0.137 & \cellcolor[HTML]{F2F3F4}0.137 & \cellcolor[HTML]{F2F3F4}0.137 & \cellcolor[HTML]{F2F3F4}0.137 & \cellcolor[HTML]{F2F3F4}0.137 & \cellcolor[HTML]{F2F3F4}0.137 & \cellcolor[HTML]{F2F3F4}0.137 & \cellcolor[HTML]{F2F3F4}0.137 & \cellcolor[HTML]{F2F3F4}0.137 & \cellcolor[HTML]{F2F3F4}0.137 & \cellcolor[HTML]{F2F3F4}0.137 & \cellcolor[HTML]{F2F3F4}0.137 & \cellcolor[HTML]{F2F3F4}0.137 & \cellcolor[HTML]{F2F3F4}\cellcolor[HTML]{F2F3F4}0.138 & \cellcolor[HTML]{F2F3F4}0.138 \\
 & \multirow{2}{*}{\textsc{NetMF}} & 0.564 & 0.577 & 0.586 & 0.589 & 0.593 & 0.596 & 0.599 & 0.601 & 0.604 & 0.605 & 0.613 & 0.617 & 0.619 & 0.620 & 0.620 & 0.623 & 0.623 & 0.623 \\
 &  & 0.463 & 0.490 & 0.503 & 0.506 & 0.510 & 0.513 & 0.517 & 0.518 & 0.521 & 0.522 & 0.528 & 0.530 & 0.532 & 0.531 & 0.531 & 0.533 & 0.533 & 0.533
 \\\midrule
\multirow{6}{*}{\rotatebox{90}{\textsc{EFGE}}} & \multirow{2}{*}{\textsc{EFGE-Bern}} & 0.554 & 0.573 & 0.586 & 0.598 & 0.604 & 0.610 & 0.615 & 0.617 & 0.618 & 0.622 & 0.631 & 0.634 & 0.635 & 0.638 & 0.637 & 0.638 & 0.638 & 0.638 \\
 &  & 0.494 & 0.514 & 0.526 & 0.539 & 0.542 & 0.547 & 0.551 & 0.553 & 0.553 & 0.556 & 0.563 & 0.566 & 0.566 & 0.569 & 0.568 & 0.569 & 0.570 & 0.570 \\
 & \multirow{2}{*}{\textsc{EFGE-Pois}} & \cellcolor[HTML]{F2F3F4}0.574 & \cellcolor[HTML]{F2F3F4}0.588 & \cellcolor[HTML]{F2F3F4}0.598 & \cellcolor[HTML]{F2F3F4}0.605 & \cellcolor[HTML]{F2F3F4}0.611 & \cellcolor[HTML]{F2F3F4}0.614 & \cellcolor[HTML]{F2F3F4}0.618 &\cellcolor[HTML]{F2F3F4} 0.620 &\cellcolor[HTML]{F2F3F4} 0.623 & \cellcolor[HTML]{F2F3F4}0.624 & \cellcolor[HTML]{F2F3F4}0.631 & \cellcolor[HTML]{F2F3F4}0.635 & \cellcolor[HTML]{F2F3F4}0.636 & \cellcolor[HTML]{F2F3F4}0.637 & \cellcolor[HTML]{F2F3F4}0.638 & \cellcolor[HTML]{F2F3F4}0.636 & \cellcolor[HTML]{F2F3F4}0.638 & \cellcolor[HTML]{F2F3F4}0.638 \\
 &  &\cellcolor[HTML]{F2F3F4}0.509 &\cellcolor[HTML]{F2F3F4}0.528 & \cellcolor[HTML]{F2F3F4}0.538 &
 \cellcolor[HTML]{F2F3F4}0.544 &
 \cellcolor[HTML]{F2F3F4}0.549 & \cellcolor[HTML]{F2F3F4}0.552 & \cellcolor[HTML]{F2F3F4}0.554 &
 \cellcolor[HTML]{F2F3F4}0.556 & \cellcolor[HTML]{F2F3F4}0.559 & \cellcolor[HTML]{F2F3F4}0.559 &
 \cellcolor[HTML]{F2F3F4}0.565 &
 \cellcolor[HTML]{F2F3F4}0.568 & \cellcolor[HTML]{F2F3F4}0.569 & \cellcolor[HTML]{F2F3F4}0.569 &
 \cellcolor[HTML]{F2F3F4}0.570 &
 \cellcolor[HTML]{F2F3F4}0.569 & \cellcolor[HTML]{F2F3F4}0.570 & \cellcolor[HTML]{F2F3F4}0.570 \\
 & \multirow{2}{*}{\textsc{EFGE-Norm}} & \textbf{0.596} & \textbf{0.603} & \textbf{0.610} & \textbf{0.614} & \textbf{0.618} & \textbf{0.622} & \textbf{0.622} & \textbf{0.624} & \textbf{0.627} & \textbf{0.628} & \textbf{0.635} & \textbf{0.637} & \textbf{0.640} & \textbf{0.640} & \textbf{0.641} & \textbf{0.642} & \textbf{0.642} & \textbf{0.641} \\
 &  & \textbf{0.520} & \textbf{0.533} & \textbf{0.544} & \textbf{0.548} & \textbf{0.552} & \textbf{0.556} & \textbf{0.555} & \textbf{0.558} & \textbf{0.560} & \textbf{0.562} & 0.568 & 0.569 & \textbf{0.572} & \textbf{0.573} & 0.572 & \textbf{0.574} & 0.573 & 0.572\\\bottomrule
\end{tabular}%
}
\caption{Node classification for varying training sizes for the \textsl{DBLP} network. For each method, the first row indicates the Micro-$F_1$ scores and the second one shows the Macro-$F_1$ scores.}
\label{tab:classification_dblp_appendix}
\end{table*}

\end{document}